\documentclass[final,12pt]{arxiv} 

\usepackage[utf8]{inputenc} 
\usepackage[T1]{fontenc}    
\usepackage{hyperref}       
\usepackage{url}            
\usepackage{booktabs}       
\usepackage{amsfonts}       
\usepackage{nicefrac}       
\usepackage{microtype}      
\usepackage{xcolor}         

\usepackage{dsfont}
\usepackage{enumitem}
\usepackage[toc,page,header]{appendix}

\usepackage[english]{babel}
\usepackage{amsmath, amssymb}
\usepackage{graphicx}
\usepackage{latexsym}
\usepackage{graphicx}
\usepackage{cases}
\usepackage[mathscr]{euscript}
\usepackage{xcolor}
\usepackage{colortbl}
\usepackage{thmtools}
\usepackage{thm-restate}

\usepackage{mathtools}
\usepackage{subcaption}

\usepackage{pifont}

\usepackage{multirow}
\usepackage{diagbox}

\usepackage{MnSymbol}
\DeclareMathAlphabet\mathbb{U}{msb}{m}{n}
\usepackage{xpatch}

\usepackage{tikz}
\usetikzlibrary{shapes.geometric, arrows}

\definecolor{Gray}{gray}{0.85}
\newcolumntype{g}{>{\columncolor{Gray}}c}

\hypersetup{
  breaklinks   = true, 
  colorlinks   = true, 
  urlcolor     = blue, 
  linkcolor    = blue, 
  citecolor    = blue 
}

\def\Rset{\mathbb{R}}

\DeclareMathOperator*{\E}{\mathbb{E}}

\DeclareMathOperator*{\argmax}{\rm argmax}
\DeclareMathOperator*{\argmin}{\rm argmin}

\newcommand{\nrm}[1]{{\left\vert\kern-0.25ex\left\vert\kern-0.25ex\left\vert #1 
    \right\vert\kern-0.25ex\right\vert\kern-0.25ex\right\vert}}

\DeclarePairedDelimiter{\abs}{\lvert}{\rvert} 
\DeclarePairedDelimiter{\bracket}{[}{]}
\DeclarePairedDelimiter{\curl}{\{}{\}}
\DeclarePairedDelimiter{\paren}{(}{)}

\newcommand{\sA}{{\mathscr A}}

\newcommand{\sC}{{\mathscr C}}
\newcommand{\sD}{{\mathscr D}}
\newcommand{\sE}{{\mathscr E}}

\newcommand{\sH}{{\mathscr H}}

\newcommand{\sM}{{\mathscr M}}

\newcommand{\sX}{{\mathscr X}}
\newcommand{\sY}{{\mathscr Y}}

\newcommand{\sfL}{{\mathsf L}}

\newcommand{\h}{\widehat}
\newcommand{\ov}{\overline}
\newcommand{\uv}{\underline}

\newcommand{\e}{\epsilon}
\newcommand{\ignore}[1]{}

\newcommand{\hh}{{\sf h}}

\newcommand{\1}{\mathds{1}}
\newcommand{\ldef}{{\sfL_{\rm{def}}}}
\newcommand{\ldefsc}{{\sfL_{\rm{def}}}}
\newcommand{\lsc}{{\sfL}}
\newcommand{\num}{{n_e}}

\newcommand{\Rad}{\mathfrak R}
\newcommand{\expert}{{g}}
\newcommand{\expertexpert}{{\sf g}}

\definecolor{airforceblue}{rgb}{0.36, 0.54, 0.66}
\definecolor{antiquewhite}{rgb}{0.98, 0.92, 0.84}
\definecolor{cherryblossompink}{rgb}{1.0, 0.72, 0.77}
\definecolor{coralpink}{rgb}{0.97, 0.51, 0.47}
\definecolor{chestnut}{rgb}{0.8, 0.36, 0.36}
\definecolor{darkchampagne}{rgb}{0.76, 0.7, 0.5}
\definecolor{pinegreen}{rgb}{0.0, 0.47, 0.44}

\tikzstyle{startstop1} = [rectangle, rounded corners, 
minimum width=1.6cm, 
minimum height=1cm,
text centered, 
draw=black, 
fill opacity=0.9,
text opacity=1,
fill=chestnut]

\tikzstyle{startstop2} = [rectangle, rounded corners, 
minimum width=1.6cm, 
minimum height=1cm,
text centered, 
draw=black, 
fill opacity=0.2,
text opacity=1,
fill=chestnut]

\tikzstyle{io} = [trapezium, 
trapezium stretches=true, 
trapezium left angle=70, 
trapezium right angle=110, 
text centered, 
text width=0.8cm, 
draw=black, fill=airforceblue]

\tikzstyle{process} = [rectangle, 
text centered, 
text width=1.5cm, 
draw=black, 
fill=antiquewhite]

\tikzstyle{decision1} = [diamond, aspect=2,
text centered, 
draw=black, 
fill opacity=0.9,
text opacity=1,
fill=pinegreen]

\tikzstyle{decision2} = [diamond, aspect=2,
text centered, 
draw=black, 
fill opacity=0.2,
text opacity=1,
fill=pinegreen]

\tikzstyle{arrow} = [thick,->,>=stealth]

\title[Principled Approaches for Learning to Defer with Multiple Experts]{Principled Approaches for Learning to Defer with Multiple Experts}
\usepackage{times}

\arxivauthor{%
 \Name{Anqi Mao} \Email{aqmao@cims.nyu.edu}\\
 \addr Courant Institute of Mathematical Sciences, New York%
 \AND
 \Name{Mehryar Mohri} \Email{mohri@google.com}\\
 \addr Google Research and Courant Institute of Mathematical Sciences, New York%
 \AND
 \Name{Yutao Zhong} \Email{yutao@cims.nyu.edu}\\
 \addr Courant Institute of Mathematical Sciences, New York%
}

\begin{document}

\maketitle

\begin{abstract}

We present a study of surrogate losses and algorithms for the general
problem of \emph{learning to defer with multiple experts}.  We first
introduce a new family of surrogate losses specifically tailored for
the multiple-expert setting, where the prediction and deferral
functions are learned simultaneously. We then prove that these
surrogate losses benefit from strong $\sH$-consistency bounds. We
illustrate the application of our analysis through several examples of
practical surrogate losses, for which we give explicit guarantees.
These loss functions readily lead to the design of new learning to
defer algorithms based on their minimization. While the main focus of this work is a theoretical analysis, we also report the results of several experiments on SVHN and CIFAR-10 datasets. \ignore{We present the results of several experiments with multiple datasets demonstrating the
effectiveness of these algorithms.}

\end{abstract}

  

\section{Introduction}
\label{sec:intro}

In many real-world applications, expert decisions can complement or
significantly enhance existing models. These experts may consist of
humans possessing domain expertise or more sophisticated albeit
expensive models. For instance, contemporary language models and
dialog-based text generation systems have exhibited susceptibility to
generating erroneous information, often referred to as
\emph{hallucinations}.  Thus, their response quality can be
substantially improved by deferring uncertain predictions to more
advanced or domain-specific pre-trained models. This particular issue
has been recognized as a central challenge for large language models
(LLMs) \citep{WeiEtAl2022,bubeck2023sparks}.  Similar observations
apply to other generation systems, including image or video
generation, as well as learning models used in various applications
such as image classification, image annotation, and speech
recognition. Thus, the problem of \emph{learning to defer with
multiple experts} has become increasingly critical in applications.

The concept of \emph{learning to defer} can be traced back to the
original work on \emph{learning with rejection} or \emph{abstention}
based on confidence thresholds
\citep{Chow1957,chow1970optimum,HerbeiWegkamp2005,bartlett2008classification,grandvalet2008support,yuan2010classification,WegkampYuan2011,ramaswamy2018consistent,NiCHS19},
rejection or abstention functions
\citep{CortesDeSalvoMohri2016,CortesDeSalvoMohri2016bis,charoenphakdee2021classification,caogeneralizing,CortesDeSalvoMohri2023,cheng2023regression,MaoMohriZhong2024predictor,MaoMohriZhong2024score,MohriAndorChoiCollinsMaoZhong2024learning,li2024no}, or
\emph{selective classification} \citep{el2010foundations,
  wiener2011agnostic,geifman2017selective,
  gangrade2021selective,geifman2019selectivenet}, and other methods
\citep{kalai2012reliable,
  ziyin2019deep,acar2020budget}. In these studies, either the cost of
abstention is not explicit or it is chosen to be a constant.

However, a constant cost does not fully capture all the relevant
information in the deferral scenario. It is important to take into
account the quality of the expert, whose prediction we rely on.  These
may be human experts as in several critical applications
\citep{kamar2012combining,tan2018investigating,kleinberg2018human,
  bansal2021most}. To address this gap, \citet{madras2018learning}
incorporated the human expert's decision into the cost and proposed
the first \emph{learning to defer (L2D)} framework, which has also
been examined in \citep{raghu2019algorithmic,wilder2021learning,
  pradier2021preferential,keswani2021towards}.
\citet{mozannar2020consistent} proposed the first
\emph{Bayes-consistent}
\citep{Zhang2003,bartlett2006convexity,steinwart2007compare} surrogate
loss for L2D, and subsequent work
\citep{raman2021improving,liu2022incorporating} further improved upon
it. Another Bayes-consistent surrogate loss in L2D is the
one-versus-all loss proposed by \citet{verma2022calibrated} that is
also studied in \citep{charusaie2022sample} as a special case of a
general family of loss functions. An additional line of research
investigated post-hoc methods
\citep{okati2021differentiable,narasimhanpost}, where
\citet{okati2021differentiable} proposed an alternative optimization
method between the predictor and rejector, and \citet{narasimhanpost}
provided a correction to the surrogate losses in
\citep{mozannar2020consistent, verma2022calibrated} when they are
underfitting. Finally, L2D or its variants have been adopted or studied in various other scenarios \citep{de2020regression,straitouri2021reinforcement,
  zhao2021directing,joshi2021pre,gao2021human,
  mozannar2022teaching,liu2022incorporating,
  hemmer2023learning,narasimhan2023learning,cao2023defense,chen2024learning}.
  
All the studies mentioned so far mainly focused on learning to defer
with a single expert. Most recently, \citet{verma2023learning}
highlighted the significance of \emph{learning to defer with multiple
experts}
\citep{hemmer2022forming,keswani2021towards,kerrigan2021combining,
  straitouri2022provably,benz2022counterfactual} and extended the
surrogate loss in \citep{verma2022calibrated,mozannar2020consistent}
to accommodate the multiple-expert setting, which is currently the
only work to propose Bayes-consistent surrogate losses in this
scenario. They further showed that a mixture of experts (MoE) approach
to multi-expert L2D proposed in \citep{hemmer2022forming} is not
consistent. More recently, \citet{mao2023two} examined a two-stage scenario for learning to defer with multiple experts, which is crucial for various applications. They developed new surrogate losses for this scenario and demonstrated that these are supported by stronger consistency guarantees—specifically, $\sH$-consistency bounds as introduced below—implying their Bayes consistency. \citet{mao2024regression} first examined the setting of regression with multiple experts, where they provided novel surrogate losses and their $\sH$-consistency bounds guarantees.

Meanwhile, recent work by \citet{awasthi2022Hconsistency,
  AwasthiMaoMohriZhong2022multi} introduced new consistency
guarantees, called $\sH$-consistency bounds, which they argued are
more relevant to learning than Bayes-consistency since they are
hypothesis set-specific and non-asymptotic. $\sH$-consistency bounds
are also stronger guarantees than Bayes-consistency. They established
$\sH$-consistent bounds for common surrogate losses in standard
classification (see also \citep{mao2023cross,zheng2023revisiting,MaoMohriZhong2023characterization}). This naturally
raises the question: can we design deferral surrogate losses that
benefit from these more significant consistency guarantees?

\textbf{Our contributions.} We study the general framework of learning
to defer with multiple experts.
We first introduce a new family of surrogate losses specifically
tailored for the multiple-expert setting, where the prediction and
deferral functions are learned simultaneously
(Section~\ref{sec:general-surrogate-losses}). Next, we prove that
these surrogate losses benefit from $\sH$-consistency bounds
(Section~\ref{sec:H-consistency-bounds}).
This implies, in particular, their Bayes-consistency.
We illustrate the application of our analysis through several examples
of practical surrogate losses, for which we give explicit guarantees.
These loss functions readily lead to the design of new learning to
defer algorithms based on their minimization. 
Our $\sH$-consistency bounds incorporate a crucial term known as the
\emph{minimizability gap}.  We show that this makes them more
advantageous guarantees than bounds based on the approximation error
(Section~\ref{sec:minimizability_gaps}).
We further demonstrate that our $\sH$-consistency bounds can be used
to derive generalization bounds for the minimizer of a surrogate loss
expressed in terms of the minimizability gaps
(Section~\ref{sec:learning-bound}). While the main focus of this work is a theoretical analysis, we also report the results of several experiments with SVHN and CIFAR-10 datasets (Section~\ref{sec:experiments}).
\ignore{Finally, we present the results of several experiments with multiple
datasets demonstrating the effectiveness of these algorithms
(Section~\ref{sec:experiments}).}

We give a more detailed discussion of related work in
Appendix~\ref{app:related-work}. We start with the introduction of
preliminary definitions and notation needed for our discussion of the
problem of learning to defer with multiple experts.

\section{Preliminaries}
\label{sec:pre}

\begin{figure}[t]
    \centering
    \resizebox{\columnwidth}{!}{
    \begin{tikzpicture}[node distance=2cm]
    \node (x) [io] {input $x \in \sX$};
    \node (h) [process, right of=x, xshift=0.2cm] {hypothesis $h \in \sH$};
    
    \node (h3) [decision1, right of=h, xshift=2cm] {predict 3};
    \node (h2) [decision1, above of=h3, yshift=-0.6cm] {predict 2};
    \node (h1) [decision1, above of=h2, yshift=-0.6cm] {predict 1};
    \node (h4) [decision2, below of=h3, yshift=0.6cm] {expert $\expert_1$};
    \node (h5) [decision2, below of=h4, yshift=0.6cm] {expert $\expert_2$};
    
    \node (c1) [startstop1, right of=h1, xshift=1cm] {$\1_{1\neq y}$};
    \node (c2) [startstop1, right of=h2, xshift=1cm] {$\1_{2\neq y}$};
    \node (c3) [startstop1, right of=h3, xshift=1cm] {$\1_{3\neq y}$};
    \node (c4) [startstop2, right of=h4, xshift=1cm] {$c_1(x, y)$};
    \node (c5) [startstop2, right of=h5, xshift=1cm] {$c_2(x, y)$};
    \node (loss) [above of=c1, yshift=-1.2cm] {incur loss};
    
    \node (y) [io, right of=c3, xshift=1cm] {label $y \in \sY$};
    
    \draw [arrow] (x) -- (h);
    
    \draw [arrow] (h) -- node[sloped, anchor=center, above, yshift=-0.1cm] {$\hh(x) = 1$} (h1);
    \draw [arrow] (h) -- node[sloped, anchor=center, above, yshift=-0.1cm] {$\hh(x) = 2$} (h2);
    \draw [arrow] (h) -- node[sloped, anchor=center, above, yshift=-0.1cm] {$\hh(x) = 3$} (h3);
    \draw [arrow] (h) -- node[sloped, anchor=center, above, yshift=-0.1cm] {$\hh(x) = 4$} (h4);
    \draw [arrow] (h) -- node[sloped, anchor=center, above, yshift=-0.1cm] {$\hh(x) = 5$} (h5);
    
    \draw [arrow] (h1) -- (c1);
    \draw [arrow] (h2) -- (c2);
    \draw [arrow] (h3) -- (c3);
    \draw [arrow] (h4) -- (c4);
    \draw [arrow] (h5) -- (c5);
    
    \draw [arrow] (y) -- (c1);
    \draw [arrow] (y) -- (c2);
    \draw [arrow] (y) -- (c3);
    \draw [arrow] (y) -- (c4);
    \draw [arrow] (y) -- (c5);
    \end{tikzpicture}}    
    \caption{Illustration of the scenario of learning to defer with
      multiple experts ($n=3$ and $\num=2$).}
    \label{fig:deferral}
    \vskip -0.14in
\end{figure}
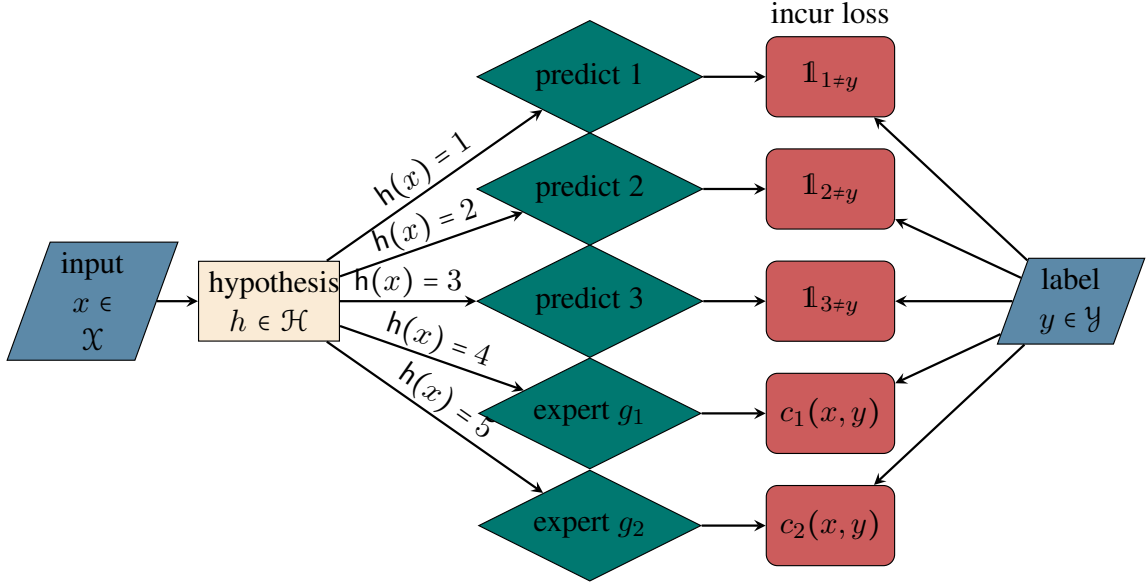

We consider the standard multi-class classification setting with an
input space $\sX$ and a set of $n \geq 2$ labels $\sY =[n]$, where we
use the notation $[n]$ to denote the set $\curl*{1, \ldots, n}$.
We study the scenario of \emph{learning to defer with multiple experts}, where the label set
$\sY$ is augmented with $\num$ additional labels $\curl*{n + 1,
  \ldots, n + \num}$ corresponding to $\num$ pre-defined experts $\expert_1,
\ldots, \expert_\num$, which are a series of functions mapping from $\sX\times \sY$ to $\Rset$.  In this scenario, the learner has the option of
returning a label $y \in \sY$, which represents the category
predicted, or a label $y = n + j$, $1 \leq j \leq \num$, in which case it is
\emph{deferring} to expert $\expert_j$.

We denote by $\ov \sY = [n + \num]$ the augmented label set and
consider a hypothesis set $ \sH$ of functions mapping from $\sX \times
\ov \sY$ to $\Rset$. The prediction associated by $h \in \sH$ to an
input $x \in \sX$ is denoted by $\hh(x)$ and defined as the element in
$\ov \sY$ with the highest score, $\hh(x) = \argmax_{y \in [n + \num]}
h(x, y)$, with an arbitrary but fixed deterministic strategy for
breaking ties. We denote by $\sH_{\rm{all}}$ the family of all
measurable functions.

The \emph{deferral loss function}
$\ldefsc$ is defined as follows for any $h \in \sH$ and $(x, y) \in
\sX \times \sY$:
\begin{equation}
\label{eq:def-score}
\ldefsc(h, x, y)
= \1_{\hh(x)\neq y} \1_{\hh(x)\in [n]}
+ \sum_{j = 1}^{\num} c_j(x, y) \1_{\hh(x) = n + j}
\end{equation}
Thus, the loss incurred coincides with the standard zero-one
classification loss when $\hh(x)$, the label predicted, is in $\sY$.
Otherwise, when $\hh(x)$ is equal to $n + j$, the loss incurred is
$c_j(x, y)$, the cost of deferring to expert $\expert_j$. We give an illustration of the scenario of learning to defer with three classes and two experts ($n=3$ and $\num=2$) in Figure~\ref{fig:deferral}. We will denote by
$\uv c_j \geq 0$ and $\ov c_j \leq 1$ finite lower and upper bounds on
the cost $c_j$, that is $c_j(x, y)\in [\uv c_j,\ov c_j]$ for all $(x,
y) \in \sX \times \sY$.
There are many possible choices for these costs. Our analysis is
general and requires no assumption other than their boundedness.
One natural choice is to define cost $c_j$ as a function relying on expert
$\expert_j$'s accuracy, for example
$c_j(x, y) = \alpha_j \1_{\expertexpert_{j}(x) \neq y} +
\beta_j$, with $\alpha_j, \beta_j > 0$, where $\expertexpert_{j}(x)=\argmax_{y\in [n]}\expert_j(x,y)$ is the
prediction made by expert $\expert_j$ for input $x$.

Given a distribution $\sD$ over $\sX \times \sY$, we will denote by
$\sE_{\ldefsc}(h)$ the expected deferral loss of a hypothesis $h \in
\sH$,
\begin{equation}
\sE_{\ldefsc}(h) = \E_{(x, y) \sim \sD}[\ldefsc(h, x, y)],
\end{equation}
and
by $\sE^*_{\ldefsc}(\sH) = \inf_{h \in \sH} \sE_{\ldefsc}(h)$ its
infimum or best-in-class expected loss. We will adopt similar
definitions for any surrogate loss function $\lsc$: 
\begin{equation}
\sE_{\lsc}(h) = \E_{(x, y) \sim \sD}[\lsc(h, x, y)],\quad\sE^*_{\lsc}(\sH) = \inf_{h \in \sH} \sE_{\lsc}(h).
\end{equation}

\section{General surrogate losses}
\label{sec:general-surrogate-losses}

In this section, we introduce a new family of surrogate losses
specifically tailored for the multiple-expert setting starting from
first principles.

The scenario we consider is one where the
prediction (first $n$ scores) and deferral functions (last $\num$
scores) are learned simultaneously.
Consider a hypothesis $h \in \sH$. Note that, for any $(x, y) \in \sX
\times \sY$, if the learner chooses to defer to an expert, $\hh(x) \in \curl*{n + 1, \ldots, n + \num}$, then it does not make a prediction of the category, and thus
$\hh(x) \neq y$\ignore{and $\hh(x) = n + j$, for some $j \in [\num]$}. This implies that
the following identity holds: 
\[\1_{\hh(x)\neq y} \1_{\hh(x) \in
  \curl*{n + 1,\ldots, n + \num}} = \1_{\hh(x) \in \curl*{n +
    1,\ldots, n + \num}}.\] 
Using this identity and $\1_{\hh(x)\in [n]} = 1 - \1_{\hh(x)
  \in \curl*{n + 1,\ldots, n + \num}}$, we can write the first term of \eqref{eq:def-score} as $\1_{\hh(x)\neq y} - \1_{\hh(x) \in \curl*{n +
    1,\ldots, n + \num}}$. Note that deferring occurs if and only if one of the experts is selected, that is $\1_{\hh(x) \in \curl*{n + 1,\ldots, n + \num}}
  = \sum_{j = 1}^{\num}\1_{\hh(x) = n + j}$. Therefore, the deferral loss function can be written in the following form for any $h \in \sH$ and $(x, y) \in \sX \times \sY$:
\begin{align*}
\ldef(h, x, y)
& = \1_{\hh(x)\neq y} - \sum_{j = 1}^{\num}\1_{\hh(x) = n + j} + \sum_{j = 1}^{\num} c_j(x, y) \1_{\hh(x) = n + j}\\
& = \1_{\hh(x)\neq y} + \sum_{j = 1}^{\num} \paren*{c_j(x, y) - 1} \1_{\hh(x) = n + j}\\
& =  \1_{\hh(x)\neq y}
+ \sum_{j = 1}^{\num} \paren*{1 - c_j(x, y)} \1_{\hh(x) \neq n + j}
+ \sum_{j = 1}^{\num} \paren*{c_j(x, y) - 1}.
\end{align*}
In light of this expression, since the last term $\sum_{j = 1}^{\num}
\paren*{c_j(x, y) - 1}$ does not depend on $h$, if $\ell$ is a
surrogate loss for the zero-one multi-class classification loss over the augmented label set $\ov \sY$, then $\lsc$, defined as follows for any $h \in \sH$ and $(x, y) \in \sX \times \sY$,
is a natural surrogate loss for $\ldefsc$:
\begin{equation}
\label{eq:sur-score}
\lsc(h, x, y)
=  \ell \paren*{h, x, y} + \sum_{j = 1}^{\num} \paren*{1 - c_j(x, y)} \,  \ell\paren*{h, x, n + j}.
\end{equation}
We will study the properties of the general family of surrogate losses
$\lsc$ thereby defined. Note that in the special case where $\ell$ is
the logistic loss and $\num = 1$, that is where there is only one
pre-defined expert, $\lsc$ coincides with the surrogate loss proposed
in \citep{mozannar2020consistent,caogeneralizing,MaoMohriZhong2024score}. However, even for
that special case, our derivation of the surrogate loss from first
principle is new and it is this analysis that enables us to define a
surrogate loss for the more general case of multiple experts and other
$\ell$ loss functions. Our formulation also recovers the softmax
surrogate loss in \citep{verma2023learning} when
$\ell=\ell_{\rm{log}}$ and $c_j(x,y)=1_{\expertexpert_j(x)\neq y}$.

\section{$\sH$-consistency bounds for surrogate losses}
\label{sec:H-consistency-bounds}

Here, we prove strong consistency guarantees for a surrogate deferral
loss $\lsc$ of the form described in the previous section, provided
that the loss function $\ell$ it is based upon admits a similar
consistency guarantee with respect to the standard zero-one
classification loss.

\textbf{$\sH$-consistency bounds.} To do so, we will adopt the notion
of \emph{$\sH$-consistency bounds} recently introduced by
\citet*{awasthi2022Hconsistency,AwasthiMaoMohriZhong2022multi} and also studied in \citep{awasthi2021calibration,awasthi2021finer,AwasthiMaoMohriZhong2023theoretically,awasthi2024dc,mao2023cross,MaoMohriZhong2023ranking,MaoMohriZhong2023rankingabs,zheng2023revisiting,MaoMohriZhong2023characterization,MaoMohriZhong2023structured,mao2024top,mao2024h}. These
are guarantees that, unlike Bayes-consistency or excess error bound,
take into account the specific hypothesis set $\sH$ and do not assume
$\sH$ to be the family of all measurable functions. Moreover, in
contrast with Bayes-consistency, they are not just asymptotic
guarantees. In this context, they have the following form:
$\sE_{\ldefsc}(h) - \sE_{\ldefsc}^*(\sH) \leq f\paren*{\sE_{\lsc}(h) -
  \sE_{\lsc}^*(\sH)}$, where $f$ is a non-decreasing function,
typically concave.  Thus, when the surrogate estimation loss
$\paren*{\sE_{\lsc}(h) - \sE_{\lsc}^*(\sH)}$ is reduced to $\e$, the
deferral estimation loss $\paren*{\sE_{\ldefsc}(h) -
  \sE_{\ldefsc}^*(\sH)}$ is guaranteed to be at most $f(\e)$.

\textbf{Minimizability gaps.} A key quantity appearing in these bounds
is the \emph{minimizability gap} $\sM_{\ell}(\sH)$ which, for a loss
function $\ell$ and hypothesis set $\sH$, measures the difference of
the best-in-class expected loss and the expected pointwise infimum of
the loss:
\[
\sM_{\ell}(\sH) = \sE^*_{\ell}(\sH) -
\E_x\bracket[\big]{\inf_{h \in \sH} \E_{y | x} \bracket*{\ell(h, x,
    y)}}.
\]
By the super-additivity of the infimum, since $\sE^*_{\ell}(\sH) =
\inf_{h \in \sH} \E_x\bracket[\big]{\E_{y | x} \bracket*{\ell(h, x,
    y)}}$, the minimizability gap is always non-negative.

When the loss function $\ell$ only depends on $h(x,\cdot)$ for all
$h$, $x$, and $y$, that is $\ell(h, x, y) = \Psi(h(x,1), \ldots,
h(x,n), y)$, for some function $\Psi$, then it is not hard to show
that the minimizability gap vanishes for the family of all measurable
functions: $\sM_{\ell}(\sH_{\rm{all}}) = 0$
\citep{steinwart2007compare}[lemma~2.5]. It is also null when
$\sE^*_{\ell}(\sH) = \sE^*_{\ell}(\sH_{\rm{all}})$, that is when
the Bayes-error coincides with the best-in-class error. In general,
however, the minimizabiliy gap is non-zero for a restricted hypothesis
set $\sH$ and is therefore important to analyze.  
In Section~\ref{sec:minimizability_gaps}, we will discuss in more
detail minimizability gaps for a relatively broad case and demonstrate
that $\sH$-consistency bounds with minimizability gaps can often be
more favorable than excess error bounds based on the approximation
error.

The following theorem is the main result of this section.

\begin{restatable}[\textbf{$\sH$-consistency bounds for
      score-based surrogates}]{theorem}{BoundScore}
\label{Thm:bound-score}
Assume that $\ell$ admits an $\sH$-consistency bound with respect to
the multi-class zero-one classification loss $\ell_{0-1}$. Thus, there
exists a non-decreasing concave function $\Gamma$ with $\Gamma(0)=0$
such that, for any distribution $\sD$ and for all $h \in \sH$, we have
\begin{equation*}
\sE_{\ell_{0-1}}(h) - \sE_{\ell_{0-1}}^*( \sH) + \sM_{\ell_{0-1}}( \sH)
\leq \Gamma\paren*{\sE_{\ell}(h)-\sE_{\ell}^*( \sH) + \sM_{\ell}(\sH)}.
\end{equation*}
Then, $\lsc$ admits the following $ \sH$-consistency bound with
respect to $\ldefsc$: for all $h\in \sH$,
\begin{equation}
\label{eq:H-consistency-bounds}
\sE_{\ldefsc}(h) - \sE_{\ldefsc}^*( \sH) + \sM_{\ldefsc}( \sH)
\leq \paren[\bigg]{\num + 1 - \sum_{j = 1}^{\num}\uv c_j} \Gamma\paren*{\frac{\sE_{\lsc}(h) - \sE_{\lsc}^*( \sH) + \sM_{\lsc}( \sH)}{\num + 1-\sum_{j = 1}^{\num}\ov c_j}}.
\end{equation}
Furthermore, constant factors $\paren*{\num + 1 - \sum_{j = 1}^{\num}\uv
  c_j}$ and $\frac{1}{\num + 1 - \sum_{j = 1}^{\num}\ov c_j}$ can be
removed when $\Gamma$ is linear.
\end{restatable}
The proof is given in Appendix~\ref{app:score}. It consists of first
analyzing the conditional regret of the deferral loss and that of a
surrogate loss. Next, we show how the former can be upper bounded in
terms of the latter by leveraging the $\sH$-consistency bound of
$\ell$ with respect to the zero-one loss with an appropriate
conditional distribution that we construct.  This, combined with the
results of \citet{AwasthiMaoMohriZhong2022multi}, proves our
$\sH$-consistency bounds.

Let us emphasize that the theorem is broadly applicable and that there
are many choices for the surrogate loss $\ell$ meeting the assumption
of the theorem: \citet{AwasthiMaoMohriZhong2022multi} showed that a
variety of surrogate loss functions $\ell$ admit an $\sH$-consistency
bound with respect to the zero-one loss for common hypothesis sets
such as linear models and multi-layer neural networks, including
\emph{sum losses} \citep{weston1998multi}, \emph{constrained losses}
\citep{lee2004multicategory}, and, as shown more recently by
\citet{mao2023cross} (see also \citep{zheng2023revisiting,MaoMohriZhong2023characterization}),
\emph{comp-sum losses}, which include the logistic loss
\citep{Verhulst1838,Verhulst1845,Berkson1944,Berkson1951}, the
\emph{sum-exponential loss} and many other loss functions.

Thus, the theorem gives a strong guarantee for a broad family of
surrogate losses $\lsc$ based upon such loss functions $\ell$.  The
presence of the minimizability gaps in these bounds is important.  In
particular, while the minimizability gap can be upper bounded by the
approximation error $\sA_{\ell}(\sH)= \sE^*_{\ell}(\sH) -
\E_x\bracket[\big]{\inf_{h \in \sH_{\rm{all}}} \E_{y | x}
  \bracket*{\ell(h, x, y)}} =
\sE^*_{\ell}(\sH)-\sE^*_{\ell}(\sH_{\rm{all}})$, it is a finer
quantity than the approximation error and can lead to more favorable
guarantees.

Note that when the Bayes-error coincides with the best-in-class error,
$\sE^*_{\sfL}(\sH) = \sE^*_{\sfL}(\sH_{\rm{all}})$, we have
$\sM_{\lsc}( \sH)\leq \sA_{\lsc}( \sH) = 0$.  This leads to the
following corollary, using the non-negativity property of the
minimizability gap.
\begin{corollary}
\label{cor:bound-score}
Assume that $\ell$ admits an $\sH$-consistency bound with respect to
the multi-class zero-one classification loss $\ell_{0-1}$. Then, for
all $h\in \sH$ and any distribution such that
$\sE^*_{\sfL}(\sH)=\sE_{\sfL}^*(\sH_{\rm{all}})$, the following bound holds:
\begin{equation*}
\sE_{\ldefsc}(h) - \sE_{\ldefsc}^*( \sH)
\leq \paren[\bigg]{\num + 1 - \sum_{j = 1}^{\num}\uv c_j} \Gamma\paren*{\frac{\sE_{\lsc}(h) - \sE_{\lsc}^*( \sH)}{\num + 1-\sum_{j = 1}^{\num}\ov c_j}},
\end{equation*}
Furthermore, constant factors $\paren*{\num + 1 - \sum_{j = 1}^{\num}\uv c_j}$ and $\frac{1}{\num + 1-\sum_{j = 1}^{\num}\ov c_j}$ can be removed 
when $\Gamma$ is linear.
\end{corollary}
Thus, when the estimation error of the surrogate loss, $\sE_{\lsc}(h)
- \sE_{\lsc}^*(\sH)$, is reduced to $\e$, the estimation error of the
deferral loss, $\sE_{\ldefsc}(h) - \sE_{\ldefsc}^*(\sH)$, is upper
bounded by \[\paren*{\num + 1 - \sum_{j = 1}^{\num}\uv c_j}
\Gamma\paren*{\e/\paren*{\num + 1 - \sum_{j = 1}^{\num}\ov c_j}}.\]
Moreover, $\sH$-consistency holds since $\sE_{\lsc}(h) -
\sE_{\lsc}^*(\sH) \to 0$ implies $\sE_{\ldefsc}(h) -
\sE_{\ldefsc}^*(\sH)\to 0$.

Table~\ref{tab:sur-score-comp} shows several examples of surrogate
deferral losses and their corresponding $\sH$-consistency bounds,
using the multi-class $\sH$-consistency bounds known for comp-sum
losses $\ell$ with respect to the zero-one loss
\citep[Theorem~1]{mao2023cross}.  The bounds have been simplified here
using the inequalities $1 \leq \num + 1 - \sum_{j = 1}^{\num}\ov
c_j\leq \num + 1 - \sum_{j = 1}^{\num}\uv c_j\leq \num + 1$. See
Appendix~\ref{app:sur-score-example-comp} for a more detailed
derivation.

Similarly, Table~\ref{tab:sur-score-sum} and
Table~\ref{tab:sur-score-cstnd} show several examples of surrogate
deferral losses with sum losses or constrained losses adopted for
$\ell$ and their corresponding $\sH$-consistency bounds, using the
multi-class $\sH$-consistency bounds in
\citep[Table~2]{AwasthiMaoMohriZhong2022multi} and
\citep[Table~3]{AwasthiMaoMohriZhong2022multi} respectively. Here too,
we present the simplified bounds by using the inequalities $1 \leq
\num + 1 - \sum_{j = 1}^{\num}\ov c_j\leq \num + 1 - \sum_{j =
  1}^{\num}\uv c_j\leq \num + 1$.  See
Appendix~\ref{app:sur-score-example-sum} and
Appendix~\ref{app:sur-score-example-cstnd} for a more detailed
derivation.

\begin{table*}[t]
  \centering
  \resizebox{\textwidth}{!}{
  \begin{tabular}{@{\hspace{0cm}}lll@{\hspace{0cm}}}
    \toprule
    $\ell$  & $\sfL$   & $\sH$-consistency bounds\\
    \midrule
    $\ell_{\rm{exp}}$ & $\sum_{y'\neq y} e^{h(x, y') - h(x, y)}+\sum_{j=1}^{\num}(1-c_j(x,y))\sum_{y'\neq n+j} e^{h(x, y') - h(x, n+j)}$ & $\sqrt{2}(\num+1)\paren*{\sE_{\lsc}(h) - \sE_{\lsc}^*( \sH)}^{\frac12}$\\
    $\ell_{\rm{log}}$   & $-\log\paren*{\frac{e^{h(x,y)}}{\sum_{y'\in \ov \sY}e^{h(x,y')}}}-\sum_{j=1}^{\num}(1-c_j(x,y))\log\paren*{\frac{e^{h(x,n+j)}}{\sum_{y'\in \ov \sY}e^{h(x,y')}}}$ & $\sqrt{2}(\num+1)\paren*{\sE_{\lsc}(h) - \sE_{\lsc}^*( \sH)}^{\frac12}$   \\
    $\ell_{\rm{gce}}$    & $\frac{1}{\alpha}\bracket*{1 - \bracket*{\frac{e^{h(x,y)}}
    {\sum_{y'\in \ov \sY} e^{h(x,y')}}}^{\alpha}}+\frac{1}{\alpha}\sum_{j=1}^{\num}(1-c_j(x,y))\bracket*{1 - \bracket*{\frac{e^{h(x,n+j)}}
    {\sum_{y'\in \ov \sY} e^{h(x,y')}}}^{\alpha}}$  & $\sqrt{2n^{\alpha}}(\num+1)\paren*{\sE_{\lsc}(h) - \sE_{\lsc}^*( \sH)}^{\frac12}$  \\
    $\ell_{\rm{mae}}$ & $ 1 - \frac{e^{h(x,y)}}{\sum_{y'\in \ov \sY} e^{h(x, y')}}+\sum_{j=1}^{\num}(1-c_j(x,y))\paren*{1 - \frac{e^{h(x,n+j)}}{\sum_{y'\in \ov \sY} e^{h(x, y')}}}$ &    $n \paren*{\sE_{\lsc}(h) - \sE_{\lsc}^*( \sH)}$    \\
    \bottomrule
  \end{tabular}
  }
  \caption{Examples of the deferral surrogate loss \eqref{eq:sur-score} with comp-sum losses adopted for $\ell$
  and their associated $\sH$-consistency bounds provided by
  Corollary~\ref{cor:bound-score} (with only the surrogate portion
  displayed).}
\label{tab:sur-score-comp}
\end{table*}

\begin{table*}[t]
  \centering
  \resizebox{\textwidth}{!}{
  \begin{tabular}{@{\hspace{0cm}}lll@{\hspace{0cm}}}
    \toprule
    $\ell$  & $\sfL$   & $\sH$-consistency bounds\\
    \midrule
    $\Phi_{\mathrm{sq}}^{\mathrm{sum}}$ & 
    $\sum_{y'\neq y} \Phi_{\rm{sq}}\paren*{\Delta_h(x,y,y')}+\sum_{j=1}^{\num}(1-c_j(x,y))\sum_{y'\neq n+j} \Phi_{\rm{sq}}\paren*{\Delta_h(x,n+j,y')}$  & $(\num+1)\paren*{\sE_{\lsc}(h) - \sE_{\lsc}^*( \sH)}^{\frac12}$\\
    $\Phi_{\mathrm{exp}}^{\mathrm{sum}}$   & $\sum_{y'\neq y} \Phi_{\rm{exp}}\paren*{\Delta_h(x,y,y')}+\sum_{j=1}^{\num}(1-c_j(x,y))\sum_{y'\neq n+j} \Phi_{\rm{exp}}\paren*{\Delta_h(x,n+j,y')}$ & $\sqrt{2}(\num+1)\paren*{\sE_{\lsc}(h) - \sE_{\lsc}^*( \sH)}^{\frac12}$   \\
    $\Phi_{\rho}^{\mathrm{sum}}$    &  $\sum_{y'\neq y} \Phi_{\rho}\paren*{\Delta_h(x,y,y')}+\sum_{j=1}^{\num}(1-c_j(x,y))\sum_{y'\neq n+j} \Phi_{\rho}\paren*{\Delta_h(x,n+j,y')}$  & $\sE_{\lsc}(h) - \sE_{\lsc}^*( \sH)$  \\
    \bottomrule
  \end{tabular}
  }
 \caption{Examples of the deferral surrogate loss \eqref{eq:sur-score} with sum losses adopted for $\ell$
  and their associated $\sH$-consistency bounds provided by
  Corollary~\ref{cor:bound-score} (with only the surrogate portion
  displayed), where $\Delta_h(x,y,y')=h(x, y) - h(x, y')$, $\Phi_{\mathrm{sq}}(t)=\max\curl*{0, 1 - t}^2$, $\Phi_{\mathrm{exp}}(t)=e^{-t}$,
and $\Phi_{\rho}(t)=\min\curl*{\max\curl*{0,1 - t/\rho},1}$.}
\label{tab:sur-score-sum}
\end{table*}

\begin{table*}[t]
  \centering
  \resizebox{\textwidth}{!}{
  \begin{tabular}{@{\hspace{0cm}}lll@{\hspace{0cm}}}
    \toprule
    $\ell$  & $\sfL$   & $\sH$-consistency bounds\\
    \midrule
    $\Phi_{\mathrm{hinge}}^{\mathrm{cstnd}}$ & $\sum_{y'\neq y}\Phi_{\mathrm{hinge}}\paren*{-h(x, y')}+\sum_{j=1}^{\num}(1-c_j(x,y))\sum_{y'\neq n+j}\Phi_{\mathrm{hinge}}\paren*{-h(x, y')}$ & $\sE_{\lsc}(h) - \sE_{\lsc}^*( \sH)$\\
    $\Phi_{\mathrm{sq}}^{\mathrm{cstnd}}$   & $\sum_{y'\neq y}\Phi_{\mathrm{sq}}\paren*{-h(x, y')}+\sum_{j=1}^{\num}(1-c_j(x,y))\sum_{y'\neq n+j}\Phi_{\mathrm{sq}}\paren*{-h(x, y')}$ & $(\num+1)\paren*{\sE_{\lsc}(h) - \sE_{\lsc}^*( \sH)}^{\frac12}$   \\
    $\Phi_{\mathrm{exp}}^{\mathrm{cstnd}}$    & $\sum_{y'\neq y}\Phi_{\mathrm{exp}}\paren*{-h(x, y')}+\sum_{j=1}^{\num}(1-c_j(x,y))\sum_{y'\neq n+j}\Phi_{\mathrm{exp}}\paren*{-h(x, y')}$  & $\sqrt{2}(\num+1)\paren*{\sE_{\lsc}(h) - \sE_{\lsc}^*( \sH)}^{\frac12}$  \\
    $\Phi_{\rho}^{\mathrm{cstnd}}$ & $\sum_{y'\neq y}\Phi_{\rho}\paren*{-h(x, y')}+\sum_{j=1}^{\num}(1-c_j(x,y))\sum_{y'\neq n+j}\Phi_{\rho}\paren*{-h(x, y')}$ &    $\sE_{\lsc}(h) - \sE_{\lsc}^*( \sH)$    \\
    \bottomrule
  \end{tabular}
  }
 \caption{Examples of the deferral surrogate loss \eqref{eq:sur-score} with constrained losses adopted for $\ell$
  and their associated $\sH$-consistency bounds provided by
  Corollary~\ref{cor:bound-score} (with only the surrogate portion
  displayed), where $\Phi_{\mathrm{hinge}}(t) = \max\curl*{0,1 - t}$, $\Phi_{\mathrm{sq}}(t)=\max\curl*{0, 1 - t}^2$,
  $\Phi_{\mathrm{exp}}(t)=e^{-t}$,
and
$\Phi_{\rho}(t)=\min\curl*{\max\curl*{0,1 - t/\rho},1}$ with the constraint that $\sum_{y\in \sY}h(x,y)=0$.}
\label{tab:sur-score-cstnd}
\end{table*}

\section{Benefits of minimizability gaps}
\label{sec:minimizability_gaps}

As already pointed out, the minimizabiliy gap can be upper bounded by
the approximation error $\sA_{\ell}(\sH)= \sE^*_{\ell}(\sH) -
\E_x\bracket[\big]{\inf_{h \in \sH_{\rm{all}}} \E_{y | x}
  \bracket*{\ell(h, x,
    y)}}=\sE^*_{\ell}(\sH)-\sE^*_{\ell}(\sH_{\rm{all}})$.  It is
however a finer quantity than the approximation error and can thus
lead to more favorable guarantees.  More precisely, as shown by
\citep{awasthi2022Hconsistency,AwasthiMaoMohriZhong2022multi}, for a
target loss function $\ell_2$ and a surrogate loss function $\ell_1$,
the excess error bound can be rewritten as
\begin{equation*}
\sE_{\ell_2} (h) - \sE^*_{\ell_2}(\sH) +\sA_{\ell_2}(\sH)
\leq \Gamma\paren*{ \sE_{\ell_1} (h) - \sE^*_{\ell_1}(\sH)+\sA_{\ell_1}(\sH)},
\end{equation*}
where $\Gamma$ is typically linear or the square-root function modulo
constants.  On the other hand, an $\sH$-consistency bound can be
expressed as follows:
\begin{equation*}
\sE_{\ell_2} (h) - \sE^*_{\ell_2}(\sH) +  \sM_{\ell_2}(\sH)  \leq \Gamma\paren*{ \sE_{\ell_1} (h) - \sE^*_{\ell_1}(\sH) + \sM_{\ell_1}(\sH)}.
\end{equation*}
For a target loss function $\ell_2$ with discrete outputs, such as the
zero-one loss or the deferral loss, we have
$\E_{x}\bracket[\big]{\inf_{h \in\sH}\E_{y | x}\bracket*{\ell_2(h, x,
    y)}}=\E_x\bracket[\big]{\inf_{h \in \sH_{\rm{all}}} \E_{y | x}
  \bracket*{\ell_2(h,x, y)}}$ when the hypothesis set generates labels
that cover all possible outcomes for each input (See
\citep[Lemma~3]{AwasthiMaoMohriZhong2022multi},
Lemma~\ref{lemma:calibration_gap_score} in
Appendix~\ref{sec:app_def_cond}). Consequently, we have
$\sM_{\ell_2}(\sH) = \sA_{\ell_2}(\sH)$. For a surrogate loss function
$\ell_1$, the minimizability gap is upper bounded by the approximation
error, $\sM_{\ell_1}(\sH)\leq \sA_{\ell_1}(\sH)$, and is generally
finer.

Consider a simple binary classification example with the conditional
distribution denoted as $\eta(x)=D(Y=1 | X=x)$. Let $\sH$ be a family
of functions $h$ such that $|h(x)| \leq \Lambda$ for all $x \in \sX$,
for some $\Lambda > 0$, and such that all values in the range
$[-\Lambda, +\Lambda]$ can be achieved. For the exponential-based
margin loss, defined as $\ell(h, x, y) = e^{-yh(x)}$, we have
\begin{equation*}
\E_{y | x}[\ell(h, x, y)] = \eta(x)
e^{-h(x)} + (1 - \eta(x)) e^{h(x)}.  
\end{equation*}
It can be observed that the infimum over all measurable functions can
be written as follows, for all $x$:
\begin{equation*}
\inf_{h
  \in \sH_{\mathrm{all}}}\E_{y | x}[\ell(h, x, y)] =
2\sqrt{\eta(x)(1-\eta(x))},
\end{equation*}
while the infimum over $\sH$, $\inf_{h \in \sH}\E_{y
  | x}[\ell(h, x, y)]$, depends on $\Lambda$. That infimum over $\sH$ is achieved by
  \begin{equation*}
   h(x)=\begin{cases}
   \min\curl*{\frac{1}{2} \log \frac{\eta(x)}{1
    -\eta(x)}, \Lambda} & \eta(x) \geq 1/2\\
     \max\curl*{\frac{1}{2}\log \frac{\eta(x)}{1 - \eta(x)}, -\Lambda} & \text{otherwise}.
   \end{cases} 
  \end{equation*}
Thus, in the deterministic case, we can explicitly compute the
difference between the approximation error and the minimizability gap:
\begin{align*}
\sA_{\ell}(\sH)-\sM_{\ell}(\sH)
= \E_{x}\bracket[\big]{\inf_{h \in\sH}\E_{y |
    x}\bracket*{\ell(h, x, y)}-\inf_{h \in
    \sH_{\rm{all}}} \E_{y | x} \bracket*{\ell(h,x, y)}}
    = e^{-\Lambda}.   
\end{align*}
As the parameter $\Lambda$ decreases, the hypothesis set $\sH$ becomes
more restricted and the difference between the approximation error and
the minimizability gap increases. In summary, an $\sH$-consistency
bound can be more favorable than the excess error bound as
$\sM_{\ell_2}(\sH) = \sA_{\ell_2}(\sH)$ when $\ell_2$ represents the
zero-one loss or deferral loss, and $\sM_{\ell_1}(\sH) \leq
\sA_{\ell_1}(\sH)$. Moreover, we will show in the next section that
our $\sH$-consistency bounds can lead to learning bounds for the
deferral loss and a hypothesis set $\sH$ with finite samples.

\section{Learning bounds}
\label{sec:learning-bound}
For a sample $S=\paren*{(x_1,y_1),\ldots,(x_m,y_m)}$ drawn from
$\sD^m$, we will denote by $\h h_S$ the empirical minimizer of the
empirical loss within $\sH$ with respect to the surrogate loss
function $\sfL$:
$
\h h_S=\argmin_{h\in \sH}\frac{1}{m}\sum_{i=1}^m \sfL(h,x_i,y_i).
$
Given an $\sH$-consistency bound in the form of
\eqref{eq:H-consistency-bounds}, we can further use it to derive a
learning bound for the deferral loss by upper bounding the surrogate
estimation error $\sE_{\lsc}(\h h_S) - \sE_{\lsc}^*(\sH)$ with the
complexity (e.g. the Rademacher complexity) of the family of functions
associated with $\sfL$ and $\sH$: $\sH_{\sfL}=\curl*{(x, y) \mapsto
  \sfL(h, x, y) \colon h \in \sH}$.

We denote by $\Rad_m^{\sfL}(\sH)$ the Rademacher complexity of
$\sH_{\sfL}$ and by $B_{\sfL}$ an upper bound of the surrogate loss
$\sfL$. Then, we obtain the following learning bound for the deferral
loss based on \eqref{eq:H-consistency-bounds}.

\begin{restatable}[\textbf{Learning bound}]{theorem}{GBoundScore}
\label{Thm:Gbound-score}
Under the same assumptions as Theorem~\ref{Thm:bound-score}, for any
$\delta > 0$, with probability at least $1-\delta$ over the draw of an
i.i.d sample $S$ of size $m$, the following deferral loss estimation
bound holds for $\h h_S$:
\begin{equation*}
\sE_{\ldefsc}(\h h_S) - \sE_{\ldefsc}^*( \sH) + \sM_{\lsc}( \sH) \\ \leq \paren[\bigg]{\num
  + 1 - \sum_{j = 1}^{\num}\uv c_j} \Gamma\paren*{\frac{4
    \Rad_m^{\sfL}(\sH) + 2 B_{\sfL} \sqrt{\tfrac{\log
        \frac{2}{\delta}}{2m}} + \sM_{\lsc}( \sH)}{\num + 1 - \sum_{j
      = 1}^{\num}\ov c_j}}.
\end{equation*}
\end{restatable}
The proof is presented in Appendix~\ref{app:Gbound-score}. To the best
of our knowledge, Theorem~\ref{Thm:Gbound-score} provides the first
finite-sample guarantee for the estimation error of the minimizer of a
surrogate deferral loss $\lsc$ defined for multiple experts.  The
proof exploits our $\sH$-consistency bounds with respect to the
deferral loss, as well as standard Rademacher complexity guarantees.

When $\uv c_j=0$ and $\ov c_j=1$ for any $j\in [\num]$, the right-hand
side of the bound admits the following simpler form:
\begin{align*}
 \paren*{\num + 1} \, \Gamma\paren*{4 \Rad_m^{\sfL}(\sH) +
   2 B_{\sfL} \sqrt{\tfrac{\log \frac{2}{\delta}}{2m}}
   + \sM_{\lsc}( \sH)}.   
\end{align*}
The dependency on the number of experts $\num$ makes this bound less
favorable. There is a trade-off however since, on the other hand, more
experts can help us achieve a better accuracy overall and reduce the
best-in-class deferral loss.  These learning bounds take into account
the minimizability gap, which varies as a function of the upper bound
$\Lambda$ on the magnitude of the scoring functions. Thus, both the
minimizability gaps and the Rademacher complexity term suggest a
regularization controlling the complexity of the hypothesis set
and the magnitude of the scores.
 
Adopting different loss functions $\ell$ in the definition of our
deferral surrogate loss \eqref{eq:sur-score} will lead to a different
functional form $\Gamma$, which can make the bound more or less
favorable. For example, a linear form of $\Gamma$ is in general more
favorable than a square-root form modulo a constant. But, the
dependency on the number of classes $n$ appearing in $\Gamma$ (e.g.,
$\ell = \ell_{\rm{gce}}$ or $\ell = \ell_{\rm{mae}}$) is also
important to take into account since a larger value of $n$ tends to
negatively impact the guarantees. We already discussed the dependency
on the number of experts $\num$ in $\Gamma$ (e.g., $\ell =
\ell_{\rm{gce}}$ or $\ell = \ell_{\rm{exp}}$) and the associated
trade-off, which is also important to consider.

Note that the bound of Theorem~\ref{Thm:Gbound-score} is expressed in
terms of the global complexity of the prediction and deferral scoring
functions $\sH$. One can however derive a finer bound distinguishing
the complexity of the deferral scoring functions and that of the
prediction scoring functions following a similar proof and analysis.

Recall that for a surrogate loss $\sfL$, the minimizability gap
$\sM_{\sfL}(\sH)$ is in general finer than the approximation error
$\sA_{\sfL}(\sH)$, while for the deferral loss, for common hypothesis
sets, these two quantities coincide. Thus, our bound can be rewritten
as follows for common hypothesis sets:
\begin{equation*}
\sE_{\ldefsc}(\h h_S) - \sE_{\ldefsc}^*(\sH_{\rm{all}})
\leq \paren[\bigg]{\num + 1 - \sum_{j = 1}^{\num}\uv c_j} \Gamma\paren*{\frac{4 \Rad_m^{\sfL}(\sH) +
2 B_{\sfL} \sqrt{\tfrac{\log \frac{2}{\delta}}{2m}} + \sM_{\lsc}( \sH)}{\num + 1-\sum_{j = 1}^{\num}\ov c_j}}.
\end{equation*}
This is more favorable and more relevant than a similar excess loss
bound where $\sM_{\lsc}( \sH)$ is replaced with $\sA_{\lsc}(\sH)$,
which could be derived from a generalization bound for the surrogate
loss.

\ignore{
as follows
and is a more meaningful learning guarantee:
\begin{align*}
\sE_{\ldefsc}(\h h_S) - \sE_{\ldefsc}^*(\sH_{\rm{all}})
\leq \paren[\bigg]{\num + 1 - \sum_{j = 1}^{\num}\uv c_j} \Gamma\paren*{\frac{4 \Rad_m^{\sfL}(\sH) +
2 B_{\sfL} \sqrt{\tfrac{\log \frac{2}{\delta}}{2m}} + \sA_{\lsc}(\sH)}{\num + 1-\sum_{j = 1}^{\num}\ov c_j}}.    
\end{align*}
}

\section{Experiments}
\label{sec:experiments}

In this section, we examine the empirical performance of our proposed
surrogate loss in the scenario of learning to defer with multiple
experts. More specifically, we aim to compare the overall system
accuracy for the learned predictor and deferral pairs, considering
varying numbers of experts. This comparison provides valuable
insights into the performance of our algorithm under different expert
configurations. We explore three different scenarios:
\begin{itemize}
  
    \item Only a single expert is available, specifically where a
      larger model than the base model is chosen as the deferral
      option.

    \item Two experts are available, consisting of one small model and
      one large model as the deferral options.

    \item Three experts are available, including one small model, one
      medium model, and one large model as the deferral options.
\end{itemize}
By comparing these scenarios, we evaluate the impact of varying the
number and type of experts on the overall system accuracy. 

\textbf{Type of cost.}  We carried out experiments with two types of
cost functions. For the first type, we selected the cost function to
be exactly the misclassification error of the expert: $c_j(x, y) =
\1_{\expertexpert_{j}(x) \neq y}$, where $\expertexpert_{j}(x) =
\argmax_{y \in [n]}\expert_j(x, y)$ is the prediction made by expert
$\expert_j$ for input $x$. In this scenario, the cost incurred for
deferring is determined solely based on the expert's accuracy. For the
second type, we chose a cost function admitting the form $c_j(x, y) =
\1_{\expertexpert_{j}(x) \neq y} + \beta_j$, where an additional
non-zero base cost $\beta_j$ is assigned to each expert.  Deferring to
a larger model then tends to incur a higher inference cost and hence,
the corresponding $\beta_j$ value for a larger model is higher as
well. In addition to the base cost, each expert also incurs a
misclassification error, as with the first type. Experimental setup and additional experiments (see Table~\ref{tab:additional}) are included in Appendix~\ref{app:experiments}.

\textbf{Experimental Results.} In Table~\ref{tab:first-type} and
Table~\ref{tab:second-type}, we report the mean and standard deviation
of the system accuracy over three runs with different random seeds. We
noticed a positive correlation between the number of experts and the
overall system accuracy. Specifically, as the number of experts
increases, the performance of the system in terms of accuracy
improves. This observation suggests that incorporating multiple
experts in the learning to defer framework can lead to better
predictions and decision-making. The results also demonstrate the
effectiveness of our proposed surrogate loss for deferral with
multiple experts.

\begin{table}[t]
  \centering
  \begin{tabular}{@{\hspace{0cm}}llll@{\hspace{0cm}}}
    \toprule
    & Single expert   & Two experts & Three experts\\
    \midrule
    SVHN  & 92.08 $\pm$ 0.15\% & 93.18 $\pm$ 0.18\%  &  93.46 $\pm$ 0.12\% \\
    CIFAR-10 & 73.31 $\pm$ 0.21\% & 77.12 $\pm$ 0.34\% & 78.71 $\pm$ 0.43\%\\
    \bottomrule
  \end{tabular}
 \caption{Overall system accuracy with the first type of cost functions.}
 \label{tab:first-type}
 \vskip -0.1in
\end{table}

\begin{table}[t]
  \centering
  \begin{tabular}{@{\hspace{0cm}}llll@{\hspace{0cm}}}
    \toprule
    & Single expert   & Two experts & Three experts\\
    \midrule
    SVHN  & 92.36 $\pm$ 0.22\% & 93.23 $\pm$ 0.21\%  &  93.36 $\pm$ 0.11\% \\
    CIFAR-10 & 73.70 $\pm$ 0.40\% & 76.29 $\pm$ 0.41\% & 76.43 $\pm$ 0.55\%\\
    \bottomrule
  \end{tabular}
  \caption{Overall system accuracy with the second type of cost functions.}
  \label{tab:second-type}
  \vskip -0.2in
\end{table}

\section{Conclusion}

We presented a comprehensive study of surrogate losses for the core
challenge of learning to defer with multiple experts. Through our
study, we established theoretical guarantees, strongly endorsing the adoption of the loss function family
we introduced. This versatile family of loss functions can effectively
facilitate the learning to defer algorithms across a wide range of
applications. Our analysis offers great flexibility by accommodating
diverse cost functions, encouraging exploration and evaluation of
various options in real-world scenarios. We encourage further research
into the theoretical properties of different choices and their impact
on the overall performance to gain deeper insights into their
effectiveness.


\bibliography{def}

\begin{thebibliography}{95}
\providecommand{\natexlab}[1]{#1}
\providecommand{\url}[1]{\texttt{#1}}
\expandafter\ifx\csname urlstyle\endcsname\relax
  \providecommand{\doi}[1]{doi: #1}\else
  \providecommand{\doi}{doi: \begingroup \urlstyle{rm}\Url}\fi

\bibitem[Acar et~al.(2020)Acar, Gangrade, and Saligrama]{acar2020budget}
Durmus Alp~Emre Acar, Aditya Gangrade, and Venkatesh Saligrama.
\newblock Budget learning via bracketing.
\newblock In \emph{International Conference on Artificial Intelligence and
  Statistics}, pages 4109--4119, 2020.

\bibitem[Awasthi et~al.(2021{\natexlab{a}})Awasthi, Frank, Mao, Mohri, and
  Zhong]{awasthi2021calibration}
Pranjal Awasthi, Natalie Frank, Anqi Mao, Mehryar Mohri, and Yutao Zhong.
\newblock Calibration and consistency of adversarial surrogate losses.
\newblock In \emph{Advances in Neural Information Processing Systems},
  2021{\natexlab{a}}.

\bibitem[Awasthi et~al.(2021{\natexlab{b}})Awasthi, Frank, and
  Mohri]{awasthi2021existence}
Pranjal Awasthi, Natalie Frank, and Mehryar Mohri.
\newblock On the existence of the adversarial bayes classifier.
\newblock In \emph{Advances in Neural Information Processing Systems}, pages
  2978--2990, 2021{\natexlab{b}}.

\bibitem[Awasthi et~al.(2021{\natexlab{c}})Awasthi, Mao, Mohri, and
  Zhong]{awasthi2021finer}
Pranjal Awasthi, Anqi Mao, Mehryar Mohri, and Yutao Zhong.
\newblock A finer calibration analysis for adversarial robustness.
\newblock \emph{arXiv preprint arXiv:2105.01550}, 2021{\natexlab{c}}.

\bibitem[Awasthi et~al.(2022{\natexlab{a}})Awasthi, Mao, Mohri, and
  Zhong]{AwasthiMaoMohriZhong2022multi}
Pranjal Awasthi, Anqi Mao, Mehryar Mohri, and Yutao Zhong.
\newblock Multi-class {${\mathscr H}$}-consistency bounds.
\newblock In \emph{Advances in neural information processing systems},
  2022{\natexlab{a}}.

\bibitem[Awasthi et~al.(2022{\natexlab{b}})Awasthi, Mao, Mohri, and
  Zhong]{awasthi2022Hconsistency}
Pranjal Awasthi, Anqi Mao, Mehryar Mohri, and Yutao Zhong.
\newblock {${\mathscr H}$}-consistency bounds for surrogate loss minimizers.
\newblock In \emph{International Conference on Machine Learning},
  2022{\natexlab{b}}.

\bibitem[Awasthi et~al.(2023)Awasthi, Mao, Mohri, and
  Zhong]{AwasthiMaoMohriZhong2023theoretically}
Pranjal Awasthi, Anqi Mao, Mehryar Mohri, and Yutao Zhong.
\newblock Theoretically grounded loss functions and algorithms for adversarial
  robustness.
\newblock In \emph{International Conference on Artificial Intelligence and
  Statistics}, pages 10077--10094, 2023.

\bibitem[Awasthi et~al.(2024)Awasthi, Mao, Mohri, and Zhong]{awasthi2024dc}
Pranjal Awasthi, Anqi Mao, Mehryar Mohri, and Yutao Zhong.
\newblock {DC}-programming for neural network optimizations.
\newblock \emph{Journal of Global Optimization}, pages 1--17, 2024.

\bibitem[Bansal et~al.(2021)Bansal, Nushi, Kamar, Horvitz, and
  Weld]{bansal2021most}
Gagan Bansal, Besmira Nushi, Ece Kamar, Eric Horvitz, and Daniel~S Weld.
\newblock Is the most accurate ai the best teammate? optimizing ai for
  teamwork.
\newblock In \emph{Proceedings of the AAAI Conference on Artificial
  Intelligence}, pages 11405--11414, 2021.

\bibitem[Bartlett and Wegkamp(2008)]{bartlett2008classification}
Peter~L Bartlett and Marten~H Wegkamp.
\newblock Classification with a reject option using a hinge loss.
\newblock \emph{Journal of Machine Learning Research}, 9\penalty0 (8), 2008.

\bibitem[Bartlett et~al.(2006)Bartlett, Jordan, and
  McAuliffe]{bartlett2006convexity}
Peter~L. Bartlett, Michael~I. Jordan, and Jon~D. McAuliffe.
\newblock Convexity, classification, and risk bounds.
\newblock \emph{Journal of the American Statistical Association}, 101\penalty0
  (473):\penalty0 138--156, 2006.

\bibitem[Benz and Rodriguez(2022)]{benz2022counterfactual}
Nina L~Corvelo Benz and Manuel~Gomez Rodriguez.
\newblock Counterfactual inference of second opinions.
\newblock In \emph{Uncertainty in Artificial Intelligence}, pages 453--463.
  PMLR, 2022.

\bibitem[Berkson(1944)]{Berkson1944}
Joseph Berkson.
\newblock Application of the logistic function to bio-assay.
\newblock \emph{Journal of the American Statistical Association}, 39:\penalty0
  357–--365, 1944.

\bibitem[Berkson(1951)]{Berkson1951}
Joseph Berkson.
\newblock Why {I} prefer logits to probits.
\newblock \emph{Biometrics}, 7\penalty0 (4):\penalty0 327–--339, 1951.

\bibitem[Bubeck et~al.(2023)Bubeck, Chandrasekaran, Eldan, Gehrke, Horvitz,
  Kamar, Lee, Lee, Li, Lundberg, et~al.]{bubeck2023sparks}
S{\'e}bastien Bubeck, Varun Chandrasekaran, Ronen Eldan, Johannes Gehrke, Eric
  Horvitz, Ece Kamar, Peter Lee, Yin~Tat Lee, Yuanzhi Li, Scott Lundberg,
  et~al.
\newblock Sparks of artificial general intelligence: Early experiments with
  gpt-4.
\newblock \emph{arXiv preprint arXiv:2303.12712}, 2023.

\bibitem[Cao et~al.(2022)Cao, Cai, Feng, Gu, Gu, An, Niu, and
  Sugiyama]{caogeneralizing}
Yuzhou Cao, Tianchi Cai, Lei Feng, Lihong Gu, Jinjie Gu, Bo~An, Gang Niu, and
  Masashi Sugiyama.
\newblock Generalizing consistent multi-class classification with rejection to
  be compatible with arbitrary losses.
\newblock In \emph{Advances in neural information processing systems}, 2022.

\bibitem[Cao et~al.(2023)Cao, Mozannar, Feng, Wei, and An]{cao2023defense}
Yuzhou Cao, Hussein Mozannar, Lei Feng, Hongxin Wei, and Bo~An.
\newblock In defense of softmax parametrization for calibrated and consistent
  learning to defer.
\newblock In \emph{Advances in Neural Information Processing Systems}, 2023.

\bibitem[Charoenphakdee et~al.(2021)Charoenphakdee, Cui, Zhang, and
  Sugiyama]{charoenphakdee2021classification}
Nontawat Charoenphakdee, Zhenghang Cui, Yivan Zhang, and Masashi Sugiyama.
\newblock Classification with rejection based on cost-sensitive classification.
\newblock In \emph{International Conference on Machine Learning}, pages
  1507--1517, 2021.

\bibitem[Charusaie et~al.(2022)Charusaie, Mozannar, Sontag, and
  Samadi]{charusaie2022sample}
Mohammad-Amin Charusaie, Hussein Mozannar, David Sontag, and Samira Samadi.
\newblock Sample efficient learning of predictors that complement humans.
\newblock In \emph{International Conference on Machine Learning}, pages
  2972--3005, 2022.

\bibitem[Chen et~al.(2024)Chen, Li, Sun, and Wang]{chen2024learning}
Guanting Chen, Xiaocheng Li, Chunlin Sun, and Hanzhao Wang.
\newblock Learning to make adherence-aware advice.
\newblock In \emph{International Conference on Learning Representations}, 2024.

\bibitem[Cheng et~al.(2023)Cheng, Cao, Wang, Wei, An, and
  Feng]{cheng2023regression}
Xin Cheng, Yuzhou Cao, Haobo Wang, Hongxin Wei, Bo~An, and Lei Feng.
\newblock Regression with cost-based rejection.
\newblock In \emph{Advances in Neural Information Processing Systems}, 2023.

\bibitem[Chow(1970)]{chow1970optimum}
C~Chow.
\newblock On optimum recognition error and reject tradeoff.
\newblock \emph{IEEE Transactions on information theory}, 16\penalty0
  (1):\penalty0 41--46, 1970.

\bibitem[Chow(1957)]{Chow1957}
C.K. Chow.
\newblock An optimum character recognition system using decision function.
\newblock \emph{IEEE T. C.}, 1957.

\bibitem[Cortes et~al.(2016{\natexlab{a}})Cortes, DeSalvo, and
  Mohri]{CortesDeSalvoMohri2016}
Corinna Cortes, Giulia DeSalvo, and Mehryar Mohri.
\newblock Learning with rejection.
\newblock In \emph{International Conference on Algorithmic Learning Theory},
  pages 67--82, 2016{\natexlab{a}}.

\bibitem[Cortes et~al.(2016{\natexlab{b}})Cortes, DeSalvo, and
  Mohri]{CortesDeSalvoMohri2016bis}
Corinna Cortes, Giulia DeSalvo, and Mehryar Mohri.
\newblock Boosting with abstention.
\newblock In \emph{Advances in Neural Information Processing Systems}, pages
  1660--1668, 2016{\natexlab{b}}.

\bibitem[Cortes et~al.(2023)Cortes, DeSalvo, and Mohri]{CortesDeSalvoMohri2023}
Corinna Cortes, Giulia DeSalvo, and Mehryar Mohri.
\newblock Theory and algorithms for learning with rejection in binary
  classification.
\newblock \emph{Annals of Mathematics and Artificial Intelligence}, pages
  1--39, 2023.

\bibitem[De et~al.(2020)De, Koley, Ganguly, and
  Gomez-Rodriguez]{de2020regression}
Abir De, Paramita Koley, Niloy Ganguly, and Manuel Gomez-Rodriguez.
\newblock Regression under human assistance.
\newblock In \emph{Proceedings of the AAAI Conference on Artificial
  Intelligence}, pages 2611--2620, 2020.

\bibitem[El-Yaniv et~al.(2010)]{el2010foundations}
Ran El-Yaniv et~al.
\newblock On the foundations of noise-free selective classification.
\newblock \emph{Journal of Machine Learning Research}, 11\penalty0 (5), 2010.

\bibitem[Gangrade et~al.(2021)Gangrade, Kag, and
  Saligrama]{gangrade2021selective}
Aditya Gangrade, Anil Kag, and Venkatesh Saligrama.
\newblock Selective classification via one-sided prediction.
\newblock In \emph{International Conference on Artificial Intelligence and
  Statistics}, pages 2179--2187, 2021.

\bibitem[Gao et~al.(2021)Gao, Saar-Tsechansky, De-Arteaga, Han, Lee, and
  Lease]{gao2021human}
Ruijiang Gao, Maytal Saar-Tsechansky, Maria De-Arteaga, Ligong Han, Min~Kyung
  Lee, and Matthew Lease.
\newblock Human-ai collaboration with bandit feedback.
\newblock \emph{arXiv preprint arXiv:2105.10614}, 2021.

\bibitem[Geifman and El-Yaniv(2017)]{geifman2017selective}
Yonatan Geifman and Ran El-Yaniv.
\newblock Selective classification for deep neural networks.
\newblock In \emph{Advances in neural information processing systems}, 2017.

\bibitem[Geifman and El-Yaniv(2019)]{geifman2019selectivenet}
Yonatan Geifman and Ran El-Yaniv.
\newblock Selectivenet: A deep neural network with an integrated reject option.
\newblock In \emph{International conference on machine learning}, pages
  2151--2159, 2019.

\bibitem[Grandvalet et~al.(2008)Grandvalet, Rakotomamonjy, Keshet, and
  Canu]{grandvalet2008support}
Yves Grandvalet, Alain Rakotomamonjy, Joseph Keshet, and St{\'e}phane Canu.
\newblock Support vector machines with a reject option.
\newblock In \emph{Advances in neural information processing systems}, 2008.

\bibitem[He et~al.(2016)He, Zhang, Ren, and Sun]{he2016deep}
Kaiming He, Xiangyu Zhang, Shaoqing Ren, and Jian Sun.
\newblock Deep residual learning for image recognition.
\newblock In \emph{Proceedings of the IEEE conference on computer vision and
  pattern recognition}, pages 770--778, 2016.

\bibitem[Hemmer et~al.(2022)Hemmer, Schellhammer, V{\"o}ssing, Jakubik, and
  Satzger]{hemmer2022forming}
Patrick Hemmer, Sebastian Schellhammer, Michael V{\"o}ssing, Johannes Jakubik,
  and Gerhard Satzger.
\newblock Forming effective human-ai teams: Building machine learning models
  that complement the capabilities of multiple experts.
\newblock \emph{arXiv preprint arXiv:2206.07948}, 2022.

\bibitem[Hemmer et~al.(2023)Hemmer, Thede, V{\"o}ssing, Jakubik, and
  K{\"u}hl]{hemmer2023learning}
Patrick Hemmer, Lukas Thede, Michael V{\"o}ssing, Johannes Jakubik, and Niklas
  K{\"u}hl.
\newblock Learning to defer with limited expert predictions.
\newblock \emph{arXiv preprint arXiv:2304.07306}, 2023.

\bibitem[Herbei and Wegkamp(2005)]{HerbeiWegkamp2005}
Radu Herbei and Marten Wegkamp.
\newblock Classification with reject option.
\newblock \emph{Can. J. Stat.}, 2005.

\bibitem[Joshi et~al.(2021)Joshi, Parbhoo, and Doshi-Velez]{joshi2021pre}
Shalmali Joshi, Sonali Parbhoo, and Finale Doshi-Velez.
\newblock Pre-emptive learning-to-defer for sequential medical decision-making
  under uncertainty.
\newblock \emph{arXiv preprint arXiv:2109.06312}, 2021.

\bibitem[Kalai et~al.(2012)Kalai, Kanade, and Mansour]{kalai2012reliable}
Adam~Tauman Kalai, Varun Kanade, and Yishay Mansour.
\newblock Reliable agnostic learning.
\newblock \emph{Journal of Computer and System Sciences}, 78\penalty0
  (5):\penalty0 1481--1495, 2012.

\bibitem[Kamar et~al.(2012)Kamar, Hacker, and Horvitz]{kamar2012combining}
Ece Kamar, Severin Hacker, and Eric Horvitz.
\newblock Combining human and machine intelligence in large-scale
  crowdsourcing.
\newblock In \emph{AAMAS}, pages 467--474, 2012.

\bibitem[Kerrigan et~al.(2021)Kerrigan, Smyth, and
  Steyvers]{kerrigan2021combining}
Gavin Kerrigan, Padhraic Smyth, and Mark Steyvers.
\newblock Combining human predictions with model probabilities via confusion
  matrices and calibration.
\newblock \emph{Advances in Neural Information Processing Systems},
  34:\penalty0 4421--4434, 2021.

\bibitem[Keswani et~al.(2021)Keswani, Lease, and
  Kenthapadi]{keswani2021towards}
Vijay Keswani, Matthew Lease, and Krishnaram Kenthapadi.
\newblock Towards unbiased and accurate deferral to multiple experts.
\newblock In \emph{Proceedings of the 2021 AAAI/ACM Conference on AI, Ethics,
  and Society}, pages 154--165, 2021.

\bibitem[Kingma and Ba(2014)]{kingma2014adam}
Diederik~P Kingma and Jimmy Ba.
\newblock Adam: A method for stochastic optimization.
\newblock \emph{arXiv preprint arXiv:1412.6980}, 2014.

\bibitem[Kleinberg et~al.(2018)Kleinberg, Lakkaraju, Leskovec, Ludwig, and
  Mullainathan]{kleinberg2018human}
Jon Kleinberg, Himabindu Lakkaraju, Jure Leskovec, Jens Ludwig, and Sendhil
  Mullainathan.
\newblock Human decisions and machine predictions.
\newblock \emph{The quarterly journal of economics}, 133\penalty0 (1):\penalty0
  237--293, 2018.

\bibitem[Krizhevsky(2009)]{Krizhevsky09learningmultiple}
Alex Krizhevsky.
\newblock Learning multiple layers of features from tiny images.
\newblock Technical report, Toronto University, 2009.

\bibitem[Kuznetsov et~al.(2014)Kuznetsov, Mohri, and Syed]{kuznetsov2014multi}
Vitaly Kuznetsov, Mehryar Mohri, and Umar Syed.
\newblock Multi-class deep boosting.
\newblock In \emph{Advances in Neural Information Processing Systems}, pages
  2501--2509, 2014.

\bibitem[Lee et~al.(2004)Lee, Lin, and Wahba]{lee2004multicategory}
Yoonkyung Lee, Yi~Lin, and Grace Wahba.
\newblock Multicategory support vector machines: Theory and application to the
  classification of microarray data and satellite radiance data.
\newblock \emph{Journal of the American Statistical Association}, 99\penalty0
  (465):\penalty0 67--81, 2004.

\bibitem[Li et~al.(2024)Li, Liu, Sun, and Wang]{li2024no}
Xiaocheng Li, Shang Liu, Chunlin Sun, and Hanzhao Wang.
\newblock When no-rejection learning is optimal for regression with rejection.
\newblock In \emph{International Conference on Artificial Intelligence and
  Statistics}, 2024.

\bibitem[Liu et~al.(2022)Liu, Gallego, and Barbieri]{liu2022incorporating}
Jessie Liu, Blanca Gallego, and Sebastiano Barbieri.
\newblock Incorporating uncertainty in learning to defer algorithms for safe
  computer-aided diagnosis.
\newblock \emph{Scientific reports}, 12\penalty0 (1):\penalty0 1762, 2022.

\bibitem[Long and Servedio(2013)]{long2013consistency}
Phil Long and Rocco Servedio.
\newblock Consistency versus realizable {H}-consistency for multiclass
  classification.
\newblock In \emph{International Conference on Machine Learning}, pages
  801--809, 2013.

\bibitem[Madras et~al.(2018)Madras, Creager, Pitassi, and
  Zemel]{madras2018learning}
David Madras, Elliot Creager, Toniann Pitassi, and Richard Zemel.
\newblock Learning adversarially fair and transferable representations.
\newblock \emph{arXiv preprint arXiv:1802.06309}, 2018.

\bibitem[Mao et~al.(2023{\natexlab{a}})Mao, Mohri, Mohri, and
  Zhong]{mao2023two}
Anqi Mao, Christopher Mohri, Mehryar Mohri, and Yutao Zhong.
\newblock Two-stage learning to defer with multiple experts.
\newblock In \emph{Advances in Neural Information Processing Systems},
  2023{\natexlab{a}}.

\bibitem[Mao et~al.(2023{\natexlab{b}})Mao, Mohri, and
  Zhong]{MaoMohriZhong2023characterization}
Anqi Mao, Mehryar Mohri, and Yutao Zhong.
\newblock {H}-consistency bounds: Characterization and extensions.
\newblock In \emph{Advances in Neural Information Processing Systems},
  2023{\natexlab{b}}.

\bibitem[Mao et~al.(2023{\natexlab{c}})Mao, Mohri, and
  Zhong]{MaoMohriZhong2023ranking}
Anqi Mao, Mehryar Mohri, and Yutao Zhong.
\newblock {H}-consistency bounds for pairwise misranking loss surrogates.
\newblock In \emph{International conference on Machine learning},
  2023{\natexlab{c}}.

\bibitem[Mao et~al.(2023{\natexlab{d}})Mao, Mohri, and
  Zhong]{MaoMohriZhong2023rankingabs}
Anqi Mao, Mehryar Mohri, and Yutao Zhong.
\newblock Ranking with abstention.
\newblock In \emph{ICML 2023 Workshop The Many Facets of Preference-Based
  Learning}, 2023{\natexlab{d}}.

\bibitem[Mao et~al.(2023{\natexlab{e}})Mao, Mohri, and
  Zhong]{MaoMohriZhong2023structured}
Anqi Mao, Mehryar Mohri, and Yutao Zhong.
\newblock Structured prediction with stronger consistency guarantees.
\newblock In \emph{Advances in Neural Information Processing Systems},
  2023{\natexlab{e}}.

\bibitem[Mao et~al.(2023{\natexlab{f}})Mao, Mohri, and Zhong]{mao2023cross}
Anqi Mao, Mehryar Mohri, and Yutao Zhong.
\newblock Cross-entropy loss functions: Theoretical analysis and applications.
\newblock In \emph{International Conference on Machine Learning},
  2023{\natexlab{f}}.

\bibitem[Mao et~al.(2024{\natexlab{a}})Mao, Mohri, and
  Zhong]{MaoMohriZhong2024predictor}
Anqi Mao, Mehryar Mohri, and Yutao Zhong.
\newblock Predictor-rejector multi-class abstention: Theoretical analysis and
  algorithms.
\newblock In \emph{International Conference on Algorithmic Learning Theory},
  2024{\natexlab{a}}.

\bibitem[Mao et~al.(2024{\natexlab{b}})Mao, Mohri, and
  Zhong]{MaoMohriZhong2024score}
Anqi Mao, Mehryar Mohri, and Yutao Zhong.
\newblock Theoretically grounded loss functions and algorithms for score-based
  multi-class abstention.
\newblock In \emph{International Conference on Artificial Intelligence and
  Statistics}, 2024{\natexlab{b}}.

\bibitem[Mao et~al.(2024{\natexlab{c}})Mao, Mohri, and Zhong]{mao2024h}
Anqi Mao, Mehryar Mohri, and Yutao Zhong.
\newblock {${\mathscr H}$}-consistency guarantees for regression.
\newblock \emph{arXiv preprint arXiv:2403.19480}, 2024{\natexlab{c}}.

\bibitem[Mao et~al.(2024{\natexlab{d}})Mao, Mohri, and
  Zhong]{mao2024regression}
Anqi Mao, Mehryar Mohri, and Yutao Zhong.
\newblock Regression with multi-expert deferral.
\newblock \emph{arXiv preprint arXiv:2403.19494}, 2024{\natexlab{d}}.

\bibitem[Mao et~al.(2024{\natexlab{e}})Mao, Mohri, and Zhong]{mao2024top}
Anqi Mao, Mehryar Mohri, and Yutao Zhong.
\newblock Top-$ k $ classification and cardinality-aware prediction.
\newblock \emph{arXiv preprint arXiv:2403.19625}, 2024{\natexlab{e}}.

\bibitem[Mohri et~al.(2024)Mohri, Andor, Choi, Collins, Mao, and
  Zhong]{MohriAndorChoiCollinsMaoZhong2024learning}
Christopher Mohri, Daniel Andor, Eunsol Choi, Michael Collins, Anqi Mao, and
  Yutao Zhong.
\newblock Learning to reject with a fixed predictor: Application to
  decontextualization.
\newblock In \emph{International Conference on Learning Representations}, 2024.

\bibitem[Mohri et~al.(2018)Mohri, Rostamizadeh, and
  Talwalkar]{MohriRostamizadehTalwalkar2018}
Mehryar Mohri, Afshin Rostamizadeh, and Ameet Talwalkar.
\newblock \emph{Foundations of Machine Learning}.
\newblock {MIT} Press, second edition, 2018.

\bibitem[Mozannar and Sontag(2020)]{mozannar2020consistent}
Hussein Mozannar and David Sontag.
\newblock Consistent estimators for learning to defer to an expert.
\newblock In \emph{International Conference on Machine Learning}, pages
  7076--7087, 2020.

\bibitem[Mozannar et~al.(2022)Mozannar, Satyanarayan, and
  Sontag]{mozannar2022teaching}
Hussein Mozannar, Arvind Satyanarayan, and David Sontag.
\newblock Teaching humans when to defer to a classifier via exemplars.
\newblock In \emph{Proceedings of the AAAI Conference on Artificial
  Intelligence}, pages 5323--5331, 2022.

\bibitem[Narasimhan et~al.(2022)Narasimhan, Jitkrittum, Menon, Rawat, and
  Kumar]{narasimhanpost}
Harikrishna Narasimhan, Wittawat Jitkrittum, Aditya~Krishna Menon, Ankit~Singh
  Rawat, and Sanjiv Kumar.
\newblock Post-hoc estimators for learning to defer to an expert.
\newblock In \emph{Advances in Neural Information Processing Systems}, pages
  29292--29304, 2022.

\bibitem[Narasimhan et~al.(2023)Narasimhan, Menon, Jitkrittum, and
  Kumar]{narasimhan2023learning}
Harikrishna Narasimhan, Aditya~Krishna Menon, Wittawat Jitkrittum, and Sanjiv
  Kumar.
\newblock Learning to reject meets ood detection: Are all abstentions created
  equal?
\newblock \emph{arXiv preprint arXiv:2301.12386}, 2023.

\bibitem[Netzer et~al.(2011)Netzer, Wang, Coates, Bissacco, Wu, and
  Ng]{Netzer2011}
Yuval Netzer, Tao Wang, Adam Coates, Alessandro Bissacco, Bo~Wu, and Andrew~Y
  Ng.
\newblock Reading digits in natural images with unsupervised feature learning.
\newblock In \emph{Advances in Neural Information Processing Systems}, 2011.

\bibitem[Ni et~al.(2019)Ni, Charoenphakdee, Honda, and Sugiyama]{NiCHS19}
Chenri Ni, Nontawat Charoenphakdee, Junya Honda, and Masashi Sugiyama.
\newblock On the calibration of multiclass classification with rejection.
\newblock In \emph{Advances in Neural Information Processing Systems}, pages
  2582--2592, 2019.

\bibitem[Okati et~al.(2021)Okati, De, and Rodriguez]{okati2021differentiable}
Nastaran Okati, Abir De, and Manuel Rodriguez.
\newblock Differentiable learning under triage.
\newblock \emph{Advances in Neural Information Processing Systems},
  34:\penalty0 9140--9151, 2021.

\bibitem[Pradier et~al.(2021)Pradier, Zazo, Parbhoo, Perlis, Zazzi, and
  Doshi-Velez]{pradier2021preferential}
Melanie~F Pradier, Javier Zazo, Sonali Parbhoo, Roy~H Perlis, Maurizio Zazzi,
  and Finale Doshi-Velez.
\newblock Preferential mixture-of-experts: Interpretable models that rely on
  human expertise as much as possible.
\newblock \emph{AMIA Summits on Translational Science Proceedings},
  2021:\penalty0 525, 2021.

\bibitem[Raghu et~al.(2019)Raghu, Blumer, Corrado, Kleinberg, Obermeyer, and
  Mullainathan]{raghu2019algorithmic}
Maithra Raghu, Katy Blumer, Greg Corrado, Jon Kleinberg, Ziad Obermeyer, and
  Sendhil Mullainathan.
\newblock The algorithmic automation problem: Prediction, triage, and human
  effort.
\newblock \emph{arXiv preprint arXiv:1903.12220}, 2019.

\bibitem[Raman and Yee(2021)]{raman2021improving}
Naveen Raman and Michael Yee.
\newblock Improving learning-to-defer algorithms through fine-tuning.
\newblock \emph{arXiv preprint arXiv:2112.10768}, 2021.

\bibitem[Ramaswamy et~al.(2018)Ramaswamy, Tewari, and
  Agarwal]{ramaswamy2018consistent}
Harish~G Ramaswamy, Ambuj Tewari, and Shivani Agarwal.
\newblock Consistent algorithms for multiclass classification with an abstain
  option.
\newblock \emph{Electronic Journal of Statistics}, 12\penalty0 (1):\penalty0
  530--554, 2018.

\bibitem[Steinwart(2007)]{steinwart2007compare}
Ingo Steinwart.
\newblock How to compare different loss functions and their risks.
\newblock \emph{Constructive Approximation}, 26\penalty0 (2):\penalty0
  225--287, 2007.

\bibitem[Straitouri et~al.(2021)Straitouri, Singla, Meresht, and
  Gomez-Rodriguez]{straitouri2021reinforcement}
Eleni Straitouri, Adish Singla, Vahid~Balazadeh Meresht, and Manuel
  Gomez-Rodriguez.
\newblock Reinforcement learning under algorithmic triage.
\newblock \emph{arXiv preprint arXiv:2109.11328}, 2021.

\bibitem[Straitouri et~al.(2022)Straitouri, Wang, Okati, and
  Rodriguez]{straitouri2022provably}
Eleni Straitouri, Lequn Wang, Nastaran Okati, and Manuel~Gomez Rodriguez.
\newblock Provably improving expert predictions with conformal prediction.
\newblock \emph{arXiv preprint arXiv:2201.12006}, 2022.

\bibitem[Tan et~al.(2018)Tan, Adebayo, Inkpen, and Kamar]{tan2018investigating}
Sarah Tan, Julius Adebayo, Kori Inkpen, and Ece Kamar.
\newblock Investigating human+ machine complementarity for recidivism
  predictions.
\newblock \emph{arXiv preprint arXiv:1808.09123}, 2018.

\bibitem[Verhulst(1838)]{Verhulst1838}
Pierre~François Verhulst.
\newblock Notice sur la loi que la population suit dans son accroissement.
\newblock \emph{Correspondance math\'ematique et physique}, 10:\penalty0
  113–--121, 1838.

\bibitem[Verhulst(1845)]{Verhulst1845}
Pierre~François Verhulst.
\newblock Recherches math\'ematiques sur la loi d'accroissement de la
  population.
\newblock \emph{Nouveaux M\'emoires de l'Acad\'emie Royale des Sciences et
  Belles-Lettres de Bruxelles}, 18:\penalty0 1–--42, 1845.

\bibitem[Verma and Nalisnick(2022)]{verma2022calibrated}
Rajeev Verma and Eric Nalisnick.
\newblock Calibrated learning to defer with one-vs-all classifiers.
\newblock In \emph{International Conference on Machine Learning}, pages
  22184--22202, 2022.

\bibitem[Verma et~al.(2023)Verma, Barrej{\'o}n, and
  Nalisnick]{verma2023learning}
Rajeev Verma, Daniel Barrej{\'o}n, and Eric Nalisnick.
\newblock Learning to defer to multiple experts: Consistent surrogate losses,
  confidence calibration, and conformal ensembles.
\newblock In \emph{International Conference on Artificial Intelligence and
  Statistics}, pages 11415--11434, 2023.

\bibitem[Wei et~al.(2022)Wei, Tay, Bommasani, Raffel, Zoph, Borgeaud, Yogatama,
  Bosma, Zhou, Metzler, Chi, Hashimoto, Vinyals, Liang, Dean, and
  Fedus]{WeiEtAl2022}
Jason Wei, Yi~Tay, Rishi Bommasani, Colin Raffel, Barret Zoph, Sebastian
  Borgeaud, Dani Yogatama, Maarten Bosma, Denny Zhou, Donald Metzler, Ed~H.
  Chi, Tatsunori Hashimoto, Oriol Vinyals, Percy Liang, Jeff Dean, and William
  Fedus.
\newblock Emergent abilities of large language models.
\newblock \emph{CoRR}, abs/2206.07682, 2022.

\bibitem[Weston and Watkins(1998)]{weston1998multi}
Jason Weston and Chris Watkins.
\newblock Multi-class support vector machines.
\newblock Technical report, Citeseer, 1998.

\bibitem[Wiener and El-Yaniv(2011)]{wiener2011agnostic}
Yair Wiener and Ran El-Yaniv.
\newblock Agnostic selective classification.
\newblock In \emph{Advances in neural information processing systems}, 2011.

\bibitem[Wilder et~al.(2021)Wilder, Horvitz, and Kamar]{wilder2021learning}
Bryan Wilder, Eric Horvitz, and Ece Kamar.
\newblock Learning to complement humans.
\newblock In \emph{International Joint Conferences on Artificial Intelligence},
  pages 1526--1533, 2021.

\bibitem[Yuan and Wegkamp(2010)]{yuan2010classification}
Ming Yuan and Marten Wegkamp.
\newblock Classification methods with reject option based on convex risk
  minimization.
\newblock \emph{Journal of Machine Learning Research}, 11\penalty0 (1), 2010.

\bibitem[Yuan and Wegkamp(2011)]{WegkampYuan2011}
Ming Yuan and Marten Wegkamp.
\newblock {SVM}s with a reject option.
\newblock In \emph{Bernoulli}, 2011.

\bibitem[Zhang and Agarwal(2020)]{zhang2020bayes}
Mingyuan Zhang and Shivani Agarwal.
\newblock Bayes consistency vs. {H}-consistency: The interplay between
  surrogate loss functions and the scoring function class.
\newblock In \emph{Advances in Neural Information Processing Systems}, 2020.

\bibitem[Zhang(2004)]{Zhang2003}
Tong Zhang.
\newblock Statistical behavior and consistency of classification methods based
  on convex risk minimization.
\newblock \emph{The Annals of Statistics}, 32\penalty0 (1):\penalty0 56--85,
  2004.

\bibitem[Zhang and Sabuncu(2018)]{zhang2018generalized}
Zhilu Zhang and Mert Sabuncu.
\newblock Generalized cross entropy loss for training deep neural networks with
  noisy labels.
\newblock In \emph{Advances in neural information processing systems}, 2018.

\bibitem[Zhao et~al.(2021)Zhao, Agrawal, Razavi, and Sontag]{zhao2021directing}
Jason Zhao, Monica Agrawal, Pedram Razavi, and David Sontag.
\newblock Directing human attention in event localization for clinical timeline
  creation.
\newblock In \emph{Machine Learning for Healthcare Conference}, pages 80--102,
  2021.

\bibitem[Zheng et~al.(2023)Zheng, Wu, Bao, Cao, Li, and
  Zhu]{zheng2023revisiting}
Chenyu Zheng, Guoqiang Wu, Fan Bao, Yue Cao, Chongxuan Li, and Jun Zhu.
\newblock Revisiting discriminative vs. generative classifiers: Theory and
  implications.
\newblock \emph{arXiv preprint arXiv:2302.02334}, 2023.

\bibitem[Ziyin et~al.(2019)Ziyin, Wang, Liang, Salakhutdinov, Morency, and
  Ueda]{ziyin2019deep}
Liu Ziyin, Zhikang Wang, Paul~Pu Liang, Ruslan Salakhutdinov, Louis-Philippe
  Morency, and Masahito Ueda.
\newblock Deep gamblers: Learning to abstain with portfolio theory.
\newblock \emph{arXiv preprint arXiv:1907.00208}, 2019.

\end{thebibliography}

\newpage
\appendix

\renewcommand{\contentsname}{Contents of Appendix}
\tableofcontents
\addtocontents{toc}{\protect\setcounter{tocdepth}{3}} 
\clearpage

\section{Related work}
\label{app:related-work}

The concept of \emph{learning to defer} has its roots in research on
abstention, particularly in binary classification scenarios with a
constant cost function. Early work by \citep{Chow1957} and
\citet{chow1970optimum} focused on rejection and set the foundation
for subsequent studies on learning with abstention. These studies
explored different approaches such as \emph{confidence-based methods}
\citep{HerbeiWegkamp2005,bartlett2008classification,
grandvalet2008support, yuan2010classification}, the
\emph{predictor-rejector framework} \citep{CortesDeSalvoMohri2016,CortesDeSalvoMohri2016bis,CortesDeSalvoMohri2023}, or \emph{selective classification}
\citep{el2010foundations, WegkampYuan2011,wiener2011agnostic}

\citet{CortesDeSalvoMohri2016,CortesDeSalvoMohri2016bis,CortesDeSalvoMohri2023} showed that
the confidence-based approach could fail to determine the optimal
rejection region when the predictor did not match the Bayes
solution. Instead, they proposed a novel \emph{predictor-rejector}
framework, for which they gave both Bayes-consistent and
\emph{realizable $\sH$-consistent}
surrogate losses \citep{long2013consistency,kuznetsov2014multi,zhang2020bayes}, which achieve state-of-the-art performance in the
binary setting.

\citet{el2010foundations,wiener2011agnostic} introduced and studied a
selective classification based on a predictor and a selector and
explored the trade-off between classifier coverage and accuracy,
drawing connections to active learning in their analysis.

The confidence-based and predictor-rejector frameworks have been both
further analyzed in the context of \emph{multi-class
classification}. \citet{ramaswamy2018consistent,NiCHS19,
geifman2017selective,acar2020budget,gangrade2021selective} extended
the confidence-based method to multi-class settings, while
\citet{NiCHS19} noted that deriving a Bayes-consistent surrogate loss
under the \emph{predictor-rejector} framework is quite challenging and
left it as an open problem. More recently,  \citet{MaoMohriZhong2024predictor} presented a series of new theoretical and algorithmic results in this framework, positively resolving this open problem. \citet{MohriAndorChoiCollinsMaoZhong2024learning} also examined this framework in the context of learning with a fixed predictor and applied their new algorithms to the task of decontextualization.
\citet{cheng2023regression,li2024no} further studied this framework in the context of regression with abstention.
In response to the challenge in the multi-class setting, \citet{mozannar2020consistent}
formulated a different \emph{score-based} approach to learn the
predictor and rejector simultaneously, by introducing an additional
scoring function corresponding to rejection. This method has been
further explored in a subsequent work \citep{caogeneralizing,MaoMohriZhong2024score}. The
surrogate losses derived under this framework are currently the state-of-the-art \citep{mozannar2020consistent,caogeneralizing,MaoMohriZhong2024score}.

\citet{geifman2019selectivenet} proposed a new neural network
architecture for abstention in the selective classification framework
for multi-class classification. They did not derive consistent
surrogate losses for this formulation.  \citet{ziyin2019deep} defined
a loss function for the predictor-selector framework based on the
doubling rate of gambling that requires almost no modification to the
model architecture.

Another line of research studied multi-class abstention using an
\emph{implicit criterion}
\citep{kalai2012reliable,acar2020budget,gangrade2021selective,
charoenphakdee2021classification}, by directly modeling regions with
high confidence.

However, a constant cost does not fully capture all the relevant
information in the deferral scenario. It is important to take into
account the quality of the expert, whose prediction we rely on.  These
may be human experts as in several critical applications
\citep{kamar2012combining,tan2018investigating,kleinberg2018human,
  bansal2021most}. To address this gap, \citet{madras2018learning}
incorporated the human expert's decision into the cost and proposed
the first \emph{learning to defer (L2D)} framework, which has also
been examined in \citep{raghu2019algorithmic,wilder2021learning,
  pradier2021preferential,keswani2021towards}.
\citet{mozannar2020consistent} proposed the first
\emph{Bayes-consistent}
\citep{Zhang2003,bartlett2006convexity,steinwart2007compare} surrogate
loss for L2D, and subsequent work
\citep{raman2021improving,liu2022incorporating} further improved upon
it. Another Bayes-consistent surrogate loss in L2D is the
one-versus-all loss proposed by \citet{verma2022calibrated} that is
also studied in \citep{charusaie2022sample} as a special case of a
general family of loss functions. An additional line of research
investigated post-hoc methods
\citep{okati2021differentiable,narasimhanpost}, where
\citet{okati2021differentiable} proposed an alternative optimization
method between the predictor and rejector, and \citet{narasimhanpost}
provided a correction to the surrogate losses in
\citep{mozannar2020consistent, verma2022calibrated} when they are
underfitting. Finally, L2D or its variants have been adopted or
studied in various other scenarios
\citep{de2020regression,straitouri2021reinforcement,
  zhao2021directing,joshi2021pre,gao2021human,
  mozannar2022teaching,liu2022incorporating,
  hemmer2023learning,narasimhan2023learning,cao2023defense,chen2024learning}.

All the studies mentioned so far mainly focused on learning to defer
with a single expert. Most recently, \citet{verma2023learning}
highlighted the significance of \emph{learning to defer with multiple
experts}
\citep{hemmer2022forming,keswani2021towards,kerrigan2021combining,
  straitouri2022provably,benz2022counterfactual} and extended the
surrogate loss in \citep{verma2022calibrated,mozannar2020consistent}
to accommodate the multiple-expert setting, which is currently the
only work to propose Bayes-consistent surrogate losses in this
scenario. They further showed that a mixture of experts (MoE) approach
to multi-expert L2D proposed in \citep{hemmer2022forming} is not
consistent.  More recently, \citet{mao2023two} examined a two-stage scenario for learning to defer with multiple experts, which is crucial for various applications. They developed new surrogate losses for this scenario and demonstrated that these are supported by stronger consistency guarantees—specifically, $\sH$-consistency bounds as introduced below—implying their Bayes consistency. \citet{mao2024regression} first examined the setting of regression with multiple experts, where they provided novel surrogate losses and their $\sH$-consistency bounds guarantees.

Meanwhile, recent work by \citet{awasthi2022Hconsistency,
  AwasthiMaoMohriZhong2022multi} introduced new consistency
guarantees, called $\sH$-consistency bounds, which they argued are
more relevant to learning than Bayes-consistency since they are
hypothesis set-specific and non-asymptotic. $\sH$-consistency bounds
are also stronger guarantees than Bayes-consistency. They established
$\sH$-consistent bounds for common surrogate losses in standard
classification (see also \citep{mao2023cross,zheng2023revisiting,MaoMohriZhong2023characterization}). These guarantees have also been used in the study of ranking \citep{MaoMohriZhong2023ranking,MaoMohriZhong2023rankingabs}, structured prediction \citep{MaoMohriZhong2023structured}, regression \citep{mao2024regression}, top-$k$ classification \citep{mao2024top} and adversarial robustness \citep{awasthi2021calibration,awasthi2021finer,awasthi2021existence,awasthi2024dc,AwasthiMaoMohriZhong2023theoretically,mao2023cross}.

In this work, we study the general framework of learning to defer
with multiple experts. Furthermore, we design deferral surrogate
losses that benefit from these more significant consistency
guarantees, namely, $\sH$-consistency bounds, in the general
multiple-expert setting.

\section{Experimental details} 
\label{app:experiments}
\paragraph{Experimental setup.} For our experiments, we used two popular
datasets: CIFAR-10 \citep{Krizhevsky09learningmultiple} and SVHN
(Street View House Numbers) \citep{Netzer2011}. CIFAR-10 consists of
$60\mathord,000$ color images in 10 different classes, with
$6\mathord,000$ images per class. The dataset is split into
$50\mathord,000$ training images and $10\mathord,000$ test
images. SVHN contains images of house numbers captured from Google
Street View. It consists of $73\mathord,257$ images for training and
$26\mathord,032$ images for testing. We trained for 50 epochs on
CIFAR-10 and 15 epochs on SVHN without any data augmentation.

In our experiments, we adopted the ResNet \citep{he2016deep}
architecture for the base model and selected various sizes of ResNet
models as experts in each scenario. Throughout all three scenarios, we
used ResNet-$4$ for both the predictor and the deferral models. In the
first scenario, we chose ResNet-$10$ as the expert model. In the
second scenario, we included ResNet-$10$ and ResNet-$16$ as expert
models. The third scenario involves ResNet-$10$, ResNet-$16$, and
ResNet-$28$ as expert models with increasing complexity. The expert
models are pre-trained on the training data of SVHN and CIFAR-10
respectively.

During the training process, we simultaneously trained the predictor
ResNet-$4$ and the deferral model ResNet-$4$. We adopted the Adam
optimizer \citep{kingma2014adam} with a batch size of $128$ and a
weight decay of $1\times 10^{-4}$. We used our proposed deferral
surrogate loss \eqref{eq:sur-score} with the generalized cross-entropy
loss being adopted for $\ell$. As suggested by
\citet{zhang2018generalized}, we set the parameter $\alpha$ to $0.7$.

For the second type of cost functions, we set the base costs as
follows: $\beta_1=0.1$, $\beta_2=0.12$ and $\beta_3=0.14$ for the SVHN
dataset and $\beta_1=0.3$, $\beta_2=0.32$, $\beta_3=0.34$ for the
CIFAR-10 dataset, where $\beta_1$ corresponds to the cost associated
with the smallest expert model, ResNet-$10$, $\beta_2$ to that of the
medium model, ResNet-$16$, and $\beta_3$ to that of the largest expert
model, ResNet-$28$. A base cost value that is not too far from the
misclassification error of expert models encourages in practice a
reasonable amount of input instances to be deferred.  We observed that
the performance remains close for other neighboring values of base
costs.

\paragraph{Additional experiments.} Here, we share additional experimental results in an intriguing setting where multiple experts are available and each of them has a clear domain of expertise. We report below the empirical results of our proposed deferral surrogate loss and the one-vs-all (OvA) surrogate loss proposed in recent work \citep{verma2023learning}, which is the state-of-the-art surrogate loss for learning to defer with multiple experts, on CIFAR-10. In this setting, the two experts have a clear domain of expertise. The expert 1 is always correct on the first three classes, 0 to 2, and predicts uniformly at random for other classes; the expert 2 is always correct on the next three classes, 3 to 5, and generates random predictions otherwise. We train a ResNet-16 for the predictor/deferral model. 

As shown in Table~\ref{tab:additional}, our method achieves comparable system accuracy with OvA. Among the images in classes 0 to 2, only $3.57\%$ is deferred to expert 2 which predicts uniformly at random. Similarly, among the images in classes 3 to 5, only $3.33\%$ is deferred to expert 1. For the rest of the images in classes 6 to 9, the predictor decides to learn to classify them by itself and actually makes $92.88\%$ of the final predictions. This illustrates that our proposed surrogate loss is effective and comparable to the baseline.

\begin{table*}[t]
  \centering
  \resizebox{\textwidth}{!}{
  \begin{tabular}{@{\hspace{0cm}}llllllllllllll@{\hspace{0cm}}}
    \toprule
    \multirow{3}{*}{Method} & \multirow{3}{*}{System accuracy (\%)} & \multicolumn{12}{c}{Ratio of deferral (\%)} \\
    \cmidrule{3-14}
    & & \multicolumn{3}{|c|}{all the classes} & \multicolumn{3}{|c|}{classes 0 to 2} & \multicolumn{3}{|c|}{classes 3 to 5} & \multicolumn{3}{|c|}{classes 6 to 9} \\
    \cmidrule{3-14}
    & & predictor & expert 1 & expert 2 & predictor & expert 1 & expert 2 & predictor & expert 1 & expert 2 & predictor & expert 1 & expert 2\\
    \midrule
    Ours & 92.19 & 61.43 & 17.38 & 21.19 & 46.77 & 49.67 & 3.57 & 33.60 & 3.33 & 63.07 & 92.88 & 3.43 & 3.70 \\
    OvA &  91.39 & 59.72 & 16.78 & 23.50 & 48.63 & 47.67 & 3.70 & 27.87 & 2.47 & 69.67 & 92.73 & 3.50 & 3.78 \\
    \bottomrule
  \end{tabular}
  }
 \caption{Comparison of our proposed deferral surrogate loss with the one-vs-all (OvA) surrogate loss in an intriguing setting where multiple experts are available and each of them has a clear domain of expertise.}
 \label{tab:additional}
\end{table*}

\section{Proof of \texorpdfstring{$\sH$}{H}-consistency bounds for deferral surrogate losses}

To prove $\sH$-consistency bounds for our deferral surrogate loss
functions, we will show how the \emph{conditional regret} of the
deferral loss can be upper bounded in terms of the \emph{conditional
regret} of the surrogate loss. The general theorems proven by
\citet[Theorem~4, Theorem~5]{AwasthiMaoMohriZhong2022multi} then
guarantee our $\sH$-consistency bounds.

For any $x \in \sX$ and $y \in \sY$, let $p(x, y)$ denote the
conditional probability of $Y = y$ given $X = x$ for any $y\in \sY$.
Then, for any $x \in \sX$, the \emph{conditional $\ldefsc$-loss}
$\sC_{\ldefsc}(h, x)$ and \emph{conditional regret} (or \emph{calibration
gap}) $\Delta \sC_{\ldefsc}(h, x)$ of a hypothesis $h \in \sH$ are
defined by
\begin{align*}
  \sC_{\ldefsc}(h, x)
& = \E_{y | x}[\ldefsc(h, x, y)]
= \sum_{y \in \sY} p(x, y) \ldefsc(h, x, y)\\
\Delta \sC_{\ldefsc}(h, x)
& = \sC_{\ldefsc}(h, x) - \sC^*_{\ldefsc}(\sH, x),
\end{align*}
where $\sC^*_{\ldefsc}(\sH, x) = \inf_{h \in \sH} \sC_{\ldefsc}(h, x)$.
Similar definitions hold for the surrogate loss $\lsc$. To bound
$\Delta \sC_{\ldefsc}(h, x)$ in terms of
$\Delta \sC_{\lsc}(h, x)$, we first give more explicit expressions
for these conditional regrets.

To do so, it will be convenient to use the following definition
for any $x \in \sX$ and $y \in [n + \num]$:
\begin{align*}
q(x, y) =
\begin{cases}
p(x, y) & y \in \sY \\
\ignore{\E_{y | x} \bracket*{1 - c_j(x, y)} = } 1 - \sum_{y \in \sY} p(x,y) c_j(x, y) & n + 1 \le y \leq n + \num.
\end{cases}
\end{align*}
Note that $q(x, y)$ is non-negative but, in general, these quantities
do not sum to one. We denote by $\ov q(x, y) = \frac{q(x, y)}{Q}$ their
normalized counterparts which represent probabilities, where
$Q = \sum_{y\in [n + \num]} q(x, y)$. 

For any $x \in \sX$, we will denote by $\mathsf H(x)$ the set of
labels generated by hypotheses in $\sH$: $\mathsf H(x) = \curl*{
  \hh(x) \colon h \in \sH}$. We denote by $y_{\max} \in [n + \num]$
the label associated by $q$ to an input $x\in \sX$, defined as $
y_{\max} = \argmax_{y \in[n + \num]} q(x, y)$, with the same
deterministic strategy for breaking ties as that of $ \hh(x)$.

\subsection{Conditional regret of the deferral loss}
\label{sec:app_def_cond}

With these definitions, we can now express the conditional loss
and regret of the deferral loss.

\begin{restatable}{lemma}{CalibrationGapScore}
\label{lemma:calibration_gap_score}
For any $x \in \sX$,
the minimal conditional $\ldefsc$-loss and
the calibration gap for $\ldefsc$ can be expressed as follows:
\begin{align*}
\sC^*_{\ldefsc}( \sH,x) & = 1 - \max_{y\in \mathsf H(x)} q(x, y)\\
\Delta\sC_{\ldefsc, \sH}(h,x) & = \max_{y\in \mathsf H(x)} q(x, y) - q(x, \hh(x)).
\end{align*}
\end{restatable}
\begin{proof}
The conditional
$\ldefsc$-risk of $h$ can be expressed as follows:
\begin{align*}
\sC_{\ldefsc}(h, x)
& = \E_{y | x} \bracket*{\ldefsc(h, x, y)}\\
& = \E_{y | x}
  \bracket*{\1_{\hh(x)\neq y}}\1_{\hh(x)\in [n]} + \sum_{j = 1}^{\num} \E_{y | x}
  \bracket*{c_j(x, y)}\1_{\hh(x) = n + j}\\
& = \sum_{y\in \sY} q(x, y) \1_{\hh(x)\neq y} \1_{\hh(x)\in [n]} + \sum_{j = 1}^{\num} (1 - q(x, n + j)) \1_{\hh(x) = n + j}\\
& = (1 - q(x, \hh(x))) \1_{\hh(x)\in [n]} + \sum_{j = 1}^{\num} (1 - q(x, \hh(x))) \1_{\hh(x) = n + j}\\
& = 1 - q(x, \hh(x)).
\end{align*}
Then, the minimal conditional $\ldefsc$-risk is given by
\[
\sC_{\ldefsc}^*(\sH,x) = 1 - \max_{y\in \mathsf H(x)} q(x, y),
\]
and the calibration gap can be expressed as follows:
\begin{align*}
  \Delta \sC_{\ldefsc,\sH}(h, x)
 = \sC_{\ldefsc}(h, x)-\sC_{\ldefsc}^*(\sH,x) = \max_{y\in \mathsf H(x)} q(x, y)-q(x,\hh(x)),
\end{align*}
which completes the proof.
\end{proof}

\subsection{Conditional regret of a surrogate deferral loss}

\begin{restatable}{lemma}{SurrogateCalibrationGapScore}
\label{lemma:surrogate_calibration_gap_score}
For any $x \in \sX$,
the conditional surrogate $\lsc$-loss and regret 
can be expressed as follows:
\begin{align*}
  \sC_{\lsc}(h, x)
  & = \sum_{y \in [n + \num]} q(x, y) \ell(h, x, y)\\
  \Delta \sC_{\lsc}(h, x)
  & = \sum_{y \in [n + \num]} q(x, y) \ell(h, x, y)  - \inf_{h \in \sH} \sum_{y \in [n + \num]} q(x, y) \ell(h, x, y).
\end{align*}
\end{restatable}

\begin{proof}
By definition, $\sC_{\lsc}(h, x)$ is the conditional-$\lsc$ loss can
be expressed as follows:
\begin{equation}
\label{eq:cond-surrogate}
\begin{aligned}
  \sC_{\lsc}(h, x)
  & = \E_y
  \bracket*{\lsc(h, x, y) }\\
  & = \E_y
  \bracket*{\ell \paren*{h, x, y} } +\sum_{j = 1}^{\num}\E_{y | x}
  \bracket*{\paren*{1 - c_j(x, y)} }\ell\paren*{h, x, n + j}\\  
  & = \sum_{y \in \sY} q(x, y) \ell \paren*{h, x, y} + \sum_{j = 1}^{\num}  q(x,n+j)  \ell\paren*{h, x, n + j}\\  
  & = \sum_{y\in [n + \num]} q(x, y)  \ell(h, x, y),
\end{aligned}
\end{equation}
which ends the proof.
\end{proof}

\subsection{Conditional regret of zero-one loss}

We will also make use of the following result for the zero-one loss
$\ell_{0-1}(h, x, y) = \1_{\hh(x) \neq y}$ with label space $[n+\num]$
and the conditional probability vector $\ov q(x,\cdot)$, which
characterizes the minimal conditional $\ell_{0-1}$-loss and the
corresponding calibration gap
\citep[Lemma~3]{AwasthiMaoMohriZhong2022multi}.

\begin{restatable}{lemma}{ExplicitAssumptionQ}
\label{lemma:explicit_assumption_01_q}
For any $x \in \sX$,
the minimal conditional $\ell_{0-1}$-loss and
the calibration gap for $\ell_{0-1}$ can be expressed as follows:
\begin{align*}
\sC^*_{\ell_{0-1}}(x) & = 1 - \max_{y\in \mathsf H(x)} \ov q(x, y)\\
\Delta\sC_{\ell_{0-1}}(h,x) & = \max_{y\in \mathsf H(x)} \ov q(x, y) - \ov q(x,\hh(x)).
\end{align*}
\end{restatable}

\subsection{Proof of \texorpdfstring{$\sH$}{H}-consistency bounds
  for  deferral surrogate losses (Theorem~\ref{Thm:bound-score})}
\label{app:score}

\BoundScore*
\begin{proof}
We denote the normalization factor as $Q=\sum_{y\in [n+\num]}q(x,
y)=\num+1-\E_y \bracket*{c_j(x, y)}$, which is a constant that ensures
the sum of $\ov q(x, y)=\frac{q(x, y)}{Q}$ is equal to 1.  By
Lemma~\ref{lemma:calibration_gap_score}, the calibration gap of
$\ldefsc$ can be expressed and upper-bounded as follows:
\begin{align*}
\Delta \sC_{\ldefsc}(h, x)
& = \max_{y\in \mathsf H(x)} q(x, y) - q(x, \hh(x)) \tag{Lemma~\ref{lemma:calibration_gap_score}}\\
& = Q \paren*{\max_{y\in \mathsf H(x)} \ov q(x, y) -  \ov q(x,  \hh(x))}\\
& = Q \Delta \sC_{\ell_{0-1}}(h, x)\tag{Lemma~\ref{lemma:explicit_assumption_01_q}}\\
& \leq Q\Gamma\paren*{\Delta \sC_{ \ell, \sH}(h, x)} \tag{$\sH$-consistency
bound of $ \ell$}\\
& = Q\Gamma\bigg(\sum_{y\in [n + \num]} \ov q(x, y)  \ell(h,x, y) -\inf_{h \in \sH}\sum_{y\in [n + \num]} \ov q(x, y)  \ell(h,x, y)\bigg) \\
& = Q\Gamma\bigg(\sum_{y\in [n + \num]} \frac{q(x, y)}{Q}  \ell(h,x, y) -\inf_{h \in \sH}\sum_{y\in [n + \num]} \frac{q(x, y)}{Q}  \ell(h,x, y)\bigg) \\
& = Q\Gamma\paren*{\frac{1}{Q}\Delta \sC_{\lsc}(h, x)} \tag{Lemma~\ref{lemma:surrogate_calibration_gap_score}}.
\end{align*}
Thus, taking expectations gives:
\begin{align*}
\sE_{\ldefsc}(h) - \sE_{\ldefsc}^*( \sH) + \sM_{\ldefsc}( \sH)
& = \E_{X}\bracket*{\Delta \sC_{\ldefsc}(h, x)}\\
& \leq \E_X\bracket*{Q\Gamma\paren*{\frac{1}{Q}\Delta \sC_{\lsc}(h, x)}}\\
& \leq Q \Gamma\paren*{\frac{1}{Q} \E_X\bracket*{\Delta \sC_{\lsc}(h, x)}}
\tag{concavity of $\Gamma$ and Jensen's ineq.}\\
& = Q \Gamma\paren*{\frac{\sE_{\lsc}(h)-\sE_{\lsc}^*( \sH) + \sM_{\lsc}( \sH)}{Q}}\\
& = \paren*{\num+1-\E_y
  \bracket*{c_j(x, y)}} \Gamma\paren*{\frac{\sE_{\lsc}(h) - \sE_{\lsc}^*( \sH) + \sM_{\lsc}( \sH)}{\num+1-\E_y
  \bracket*{c_j(x, y)}}}\\
&\leq \paren*{\num + 1-\sum_{j = 1}^{\num}\uv c_j} \Gamma\paren*{\frac{\sE_{\lsc}(h) - \sE_{\lsc}^*( \sH) + \sM_{\lsc}( \sH)}{\num + 1 - \sum_{j = 1}^{\num}\ov c_j}}
\tag{$\uv c_j \leq c_j(x, y)\leq \ov c_j, \forall j\in [\num]$}
\end{align*}
and $\sE_{\ldefsc}(h) - \sE_{\ldefsc}^*( \sH) + \sM_{\ldefsc}( \sH)\leq \Gamma\paren*{\sE_{\lsc}(h) - \sE_{\lsc}^*( \sH) + \sM_{\lsc}( \sH)}$ when $\Gamma$ is linear, which completes the proof.
\end{proof}

\section{Examples of deferral surrogate losses and
  their \texorpdfstring{$\sH$}{H}-consistency bounds}
\label{app:sur-score-example}
\subsection{\texorpdfstring{$\ell$}{ell} being adopted as comp-sum losses}
\label{app:sur-score-example-comp}
\paragraph{Example: $\ell=\ell_{\rm{exp}}$.} Plug in $\ell=\ell_{\rm{exp}}=\sum_{y'\neq y} e^{h(x, y') - h(x, y)}$ in \eqref{eq:sur-score}, we obtain
\begin{align*}
\sfL = \sum_{y'\neq y} e^{h(x, y') - h(x, y)} +\sum_{j=1}^{\num}(1-c_j(x,y))\sum_{y'\neq n+j} e^{h(x, y') - h(x, n+j)}.
\end{align*}
By \citet[Theorem~1]{mao2023cross}, $\ell_{\rm{exp}}$  admits an $\sH$-consistency bound with respect to $\ell_{0-1}$ with $\Gamma(t)=\sqrt{2t}$, using Corollary~\ref{cor:bound-score}, we obtain
\begin{equation*}
\sE_{\ldefsc}(h) - \sE_{\ldefsc}^*( \sH)
\leq \sqrt{2}\paren[\bigg]{\num + 1 - \sum_{j = 1}^{\num}\uv c_j} \paren*{\frac{\sE_{\lsc}(h) - \sE_{\lsc}^*( \sH)}{\num + 1-\sum_{j = 1}^{\num}\ov c_j}}^{\frac12}.
\end{equation*}
Since $1 \leq \num + 1 - \sum_{j = 1}^{\num}\ov
c_j\leq \num + 1 - \sum_{j = 1}^{\num}\uv c_j\leq \num + 1$, the bound can be simplified as
\begin{equation*}
\sE_{\ldefsc}(h) - \sE_{\ldefsc}^*( \sH)
\leq \sqrt{2}(\num+1)\paren*{\sE_{\lsc}(h) - \sE_{\lsc}^*( \sH)}^{\frac12}.
\end{equation*}

\paragraph{Example: $\ell=\ell_{\rm{log}}$.} Plug in $\ell=\ell_{\rm{log}}=- \log \bracket*{\frac{e^{h(x,y)}}{\sum_{y' \in \ov\sY} e^{h(x,y')}}}$ in \eqref{eq:sur-score}, we obtain
\begin{align*}
\sfL = -\log\paren*{\frac{e^{h(x,y)}}{\sum_{y'\in \ov \sY}e^{h(x,y')}}} -\sum_{j=1}^{\num}(1-c_j(x,y))\log\paren*{\frac{e^{h(x,n+j)}}{\sum_{y'\in \ov \sY}e^{h(x,y')}}}.
\end{align*}
By \citet[Theorem~1]{mao2023cross}, $\ell_{\rm{log}}$  admits an $\sH$-consistency bound with respect to $\ell_{0-1}$ with $\Gamma(t)=\sqrt{2t}$, using Corollary~\ref{cor:bound-score}, we obtain
\begin{equation*}
\sE_{\ldefsc}(h) - \sE_{\ldefsc}^*( \sH)
\leq \sqrt{2}\paren[\bigg]{\num + 1 - \sum_{j = 1}^{\num}\uv c_j} \paren*{\frac{\sE_{\lsc}(h) - \sE_{\lsc}^*( \sH)}{\num + 1-\sum_{j = 1}^{\num}\ov c_j}}^{\frac12}.
\end{equation*}
Since $1 \leq \num + 1 - \sum_{j = 1}^{\num}\ov
c_j\leq \num + 1 - \sum_{j = 1}^{\num}\uv c_j\leq \num + 1$, the bound can be simplified as
\begin{equation*}
\sE_{\ldefsc}(h) - \sE_{\ldefsc}^*( \sH)
\leq \sqrt{2}(\num+1)\paren*{\sE_{\lsc}(h) - \sE_{\lsc}^*( \sH)}^{\frac12}.
\end{equation*}

\paragraph{Example: $\ell=\ell_{\rm{gce}}$.} Plug in $\ell=\ell_{\rm{gce}}==\frac{1}{\alpha}\bracket*{1 - \bracket*{\frac{e^{h(x,y)}}
    {\sum_{y'\in \ov\sY} e^{h(x,y')}}}^{\alpha}}$ in \eqref{eq:sur-score}, we obtain
\begin{align*}
\sfL = \frac{1}{\alpha}\bracket*{1 - \bracket*{\frac{e^{h(x,y)}}
    {\sum_{y'\in \ov \sY} e^{h(x,y')}}}^{\alpha}} +\frac{1}{\alpha}\sum_{j=1}^{\num}(1-c_j(x,y))\bracket*{1 - \bracket*{\frac{e^{h(x,n+j)}}
    {\sum_{y'\in \ov \sY} e^{h(x,y')}}}^{\alpha}}.
\end{align*}
By \citet[Theorem~1]{mao2023cross}, $\ell_{\rm{gce}}$  admits an $\sH$-consistency bound with respect to $\ell_{0-1}$ with $\Gamma(t)=\sqrt{2n^{\alpha}t}$, using Corollary~\ref{cor:bound-score}, we obtain
\begin{equation*}
\sE_{\ldefsc}(h) - \sE_{\ldefsc}^*( \sH)
\leq \sqrt{2n^{
\alpha}}\paren[\bigg]{\num + 1 - \sum_{j = 1}^{\num}\uv c_j} \paren*{\frac{\sE_{\lsc}(h) - \sE_{\lsc}^*( \sH)}{\num + 1-\sum_{j = 1}^{\num}\ov c_j}}^{\frac12}.
\end{equation*}
Since $1 \leq \num + 1 - \sum_{j = 1}^{\num}\ov
c_j\leq \num + 1 - \sum_{j = 1}^{\num}\uv c_j\leq \num + 1$, the bound can be simplified as
\begin{equation*}
\sE_{\ldefsc}(h) - \sE_{\ldefsc}^*( \sH)
\leq \sqrt{2n^{
\alpha}}(\num+1)\paren*{\sE_{\lsc}(h) - \sE_{\lsc}^*( \sH)}^{\frac12}.
\end{equation*}

\paragraph{Example: $\ell=\ell_{\rm{mae}}$.} Plug in $\ell=\ell_{\rm{mae}}=1 - \frac{e^{h(x,y)}}{\sum_{y'\in \sY} e^{h(x, y')}}$ in \eqref{eq:sur-score}, we obtain
\begin{align*}
\sfL = 1 - \frac{e^{h(x,y)}}{\sum_{y'\in \ov \sY} e^{h(x, y')}} +\sum_{j=1}^{\num}(1-c_j(x,y))\paren*{1 - \frac{e^{h(x,n+j)}}{\sum_{y'\in \ov \sY} e^{h(x, y')}}}.
\end{align*}
By \citet[Theorem~1]{mao2023cross}, $\ell_{\rm{mae}}$  admits an $\sH$-consistency bound with respect to $\ell_{0-1}$ with $\Gamma(t)=nt$, using Corollary~\ref{cor:bound-score}, we obtain
\begin{equation*}
\sE_{\ldefsc}(h) - \sE_{\ldefsc}^*( \sH)
\leq  n \paren*{\sE_{\lsc}(h) - \sE_{\lsc}^*( \sH)}.
\end{equation*}

\subsection{\texorpdfstring{$\ell$}{ell} being adopted as sum losses}
\label{app:sur-score-example-sum}
\paragraph{Example: $\ell=\Phi_{\mathrm{sq}}^{\mathrm{sum}}$.} Plug in $\ell=\Phi_{\mathrm{sq}}^{\mathrm{sum}}=\sum_{y'\neq y}\Phi_{\rm{sq}}\paren*{h(x,y)-h(x,y')}$ in \eqref{eq:sur-score}, we obtain
\begin{align*}
\sfL = \sum_{y'\neq y} \Phi_{\rm{sq}}\paren*{\Delta_h(x,y,y')} +\sum_{j=1}^{\num}(1-c_j(x,y))\sum_{y'\neq n+j} \Phi_{\rm{sq}}\paren*{\Delta_h(x,n+j,y')},
\end{align*}
where $\Delta_h(x,y,y')=h(x, y) - h(x, y')$ and $\Phi_{\mathrm{sq}}(t)=\max\curl*{0, 1 - t}^2$.
By \citet[Table~2]{AwasthiMaoMohriZhong2022multi}, $\Phi_{\mathrm{sq}}^{\mathrm{sum}}$ admits an $\sH$-consistency bound with respect to $\ell_{0-1}$ with $\Gamma(t)=\sqrt{t}$, using Corollary~\ref{cor:bound-score}, we obtain
\begin{equation*}
\sE_{\ldefsc}(h) - \sE_{\ldefsc}^*( \sH)
\leq \paren[\bigg]{\num + 1 - \sum_{j = 1}^{\num}\uv c_j} \paren*{\frac{\sE_{\lsc}(h) - \sE_{\lsc}^*( \sH)}{\num + 1-\sum_{j = 1}^{\num}\ov c_j}}^{\frac12}.
\end{equation*}
Since $1 \leq \num + 1 - \sum_{j = 1}^{\num}\ov
c_j\leq \num + 1 - \sum_{j = 1}^{\num}\uv c_j\leq \num + 1$, the bound can be simplified as
\begin{equation*}
\sE_{\ldefsc}(h) - \sE_{\ldefsc}^*( \sH)
\leq (\num+1)\paren*{\sE_{\lsc}(h) - \sE_{\lsc}^*( \sH)}^{\frac12}.
\end{equation*}

\paragraph{Example: $\ell=\Phi_{\mathrm{exp}}^{\mathrm{sum}}$.} Plug in $\ell=\Phi_{\mathrm{exp}}^{\mathrm{sum}}=\sum_{y'\neq y}\Phi_{\rm{exp}}\paren*{h(x,y)-h(x,y')}$ in \eqref{eq:sur-score}, we obtain
\begin{align*}
\sfL = \sum_{y'\neq y} \Phi_{\rm{exp}}\paren*{\Delta_h(x,y,y')} +\sum_{j=1}^{\num}(1-c_j(x,y))\sum_{y'\neq n+j} \Phi_{\rm{exp}}\paren*{\Delta_h(x,n+j,y')},
\end{align*}
where $\Delta_h(x,y,y')=h(x, y) - h(x, y')$ and $\Phi_{\mathrm{exp}}(t)=e^{-t}$.
By \citet[Table~2]{AwasthiMaoMohriZhong2022multi}, $\Phi_{\mathrm{exp}}^{\mathrm{sum}}$ admits an $\sH$-consistency bound with respect to $\ell_{0-1}$ with $\Gamma(t)=\sqrt{2t}$, using Corollary~\ref{cor:bound-score}, we obtain
\begin{equation*}
\sE_{\ldefsc}(h) - \sE_{\ldefsc}^*( \sH)
\leq \sqrt{2}\paren[\bigg]{\num + 1 - \sum_{j = 1}^{\num}\uv c_j} \paren*{\frac{\sE_{\lsc}(h) - \sE_{\lsc}^*( \sH)}{\num + 1-\sum_{j = 1}^{\num}\ov c_j}}^{\frac12}.
\end{equation*}
Since $1 \leq \num + 1 - \sum_{j = 1}^{\num}\ov
c_j\leq \num + 1 - \sum_{j = 1}^{\num}\uv c_j\leq \num + 1$, the bound can be simplified as
\begin{equation*}
\sE_{\ldefsc}(h) - \sE_{\ldefsc}^*( \sH)
\leq \sqrt{2}(\num+1)\paren*{\sE_{\lsc}(h) - \sE_{\lsc}^*( \sH)}^{\frac12}.
\end{equation*}

\paragraph{Example: $\ell=\Phi_{\rho}^{\mathrm{sum}}$.} Plug in $\ell=\Phi_{\rho}^{\mathrm{sum}}=\sum_{y'\neq y}\Phi_{\rho}\paren*{h(x,y)-h(x,y')}$ in \eqref{eq:sur-score}, we obtain
\begin{align*}
\sfL = \sum_{y'\neq y} \Phi_{\rho}\paren*{\Delta_h(x,y,y')} +\sum_{j=1}^{\num}(1-c_j(x,y))\sum_{y'\neq n+j} \Phi_{\rho}\paren*{\Delta_h(x,n+j,y')},
\end{align*}
where $\Delta_h(x,y,y')=h(x, y) - h(x, y')$ and $\Phi_{\rho}(t)=\min\curl*{\max\curl*{0,1 - t/\rho},1}$.
By \citet[Table~2]{AwasthiMaoMohriZhong2022multi}, $\Phi_{\rho}^{\mathrm{sum}}$ admits an $\sH$-consistency bound with respect to $\ell_{0-1}$ with $\Gamma(t)=t$, using Corollary~\ref{cor:bound-score}, we obtain
\begin{equation*}
\sE_{\ldefsc}(h) - \sE_{\ldefsc}^*( \sH)
\leq \sE_{\lsc}(h) - \sE_{\lsc}^*( \sH).
\end{equation*}

\subsection{\texorpdfstring{$\ell$}{ell} being adopted as constrained losses}
\label{app:sur-score-example-cstnd}
\paragraph{Example: $\ell=\Phi_{\mathrm{hinge}}^{\mathrm{cstnd}}$.} Plug in $\ell=\Phi_{\mathrm{hinge}}^{\mathrm{cstnd}}=\sum_{y'\neq y}\Phi_{\mathrm{hinge}}\paren*{-h(x, y')}$ in \eqref{eq:sur-score}, we obtain
\begin{align*}
\sfL = \sum_{y'\neq y}\Phi_{\mathrm{hinge}}\paren*{-h(x, y')} +\sum_{j=1}^{\num}(1-c_j(x,y))\sum_{y'\neq n+j}\Phi_{\mathrm{hinge}}\paren*{-h(x, y')},
\end{align*}
where $\Phi_{\mathrm{hinge}}(t) = \max\curl*{0,1 - t}$ with the constraint that $\sum_{y\in \sY}h(x,y)=0$.
By \citet[Table~3]{AwasthiMaoMohriZhong2022multi}, $\Phi_{\mathrm{hinge}}^{\mathrm{cstnd}}$ admits an $\sH$-consistency bound with respect to $\ell_{0-1}$ with $\Gamma(t)=t$, using Corollary~\ref{cor:bound-score}, we obtain
\begin{equation*}
\sE_{\ldefsc}(h) - \sE_{\ldefsc}^*( \sH)
\leq \sE_{\lsc}(h) - \sE_{\lsc}^*( \sH).
\end{equation*}

\paragraph{Example: $\ell=\Phi_{\mathrm{sq}}^{\mathrm{cstnd}}$.} Plug in $\ell=\Phi_{\mathrm{sq}}^{\mathrm{cstnd}}=\sum_{y'\neq y}\Phi_{\mathrm{sq}}\paren*{-h(x, y')}$ in \eqref{eq:sur-score}, we obtain
\begin{align*}
\sfL = \sum_{y'\neq y}\Phi_{\mathrm{sq}}\paren*{-h(x, y')} +\sum_{j=1}^{\num}(1-c_j(x,y))\sum_{y'\neq n+j}\Phi_{\mathrm{sq}}\paren*{-h(x, y')},
\end{align*}
where $\Phi_{\mathrm{sq}}(t) = \max\curl*{0, 1 - t}^2$ with the constraint that $\sum_{y\in \sY}h(x,y)=0$.
By \citet[Table~3]{AwasthiMaoMohriZhong2022multi}, $\Phi_{\mathrm{sq}}^{\mathrm{cstnd}}$ admits an $\sH$-consistency bound with respect to $\ell_{0-1}$ with $\Gamma(t)=\sqrt{t}$, using Corollary~\ref{cor:bound-score}, we obtain
\begin{equation*}
\sE_{\ldefsc}(h) - \sE_{\ldefsc}^*( \sH)
\leq \paren[\bigg]{\num + 1 - \sum_{j = 1}^{\num}\uv c_j} \paren*{\frac{\sE_{\lsc}(h) - \sE_{\lsc}^*( \sH)}{\num + 1-\sum_{j = 1}^{\num}\ov c_j}}^{\frac12}.
\end{equation*}
Since $1 \leq \num + 1 - \sum_{j = 1}^{\num}\ov
c_j\leq \num + 1 - \sum_{j = 1}^{\num}\uv c_j\leq \num + 1$, the bound can be simplified as
\begin{equation*}
\sE_{\ldefsc}(h) - \sE_{\ldefsc}^*( \sH)
\leq (\num+1)\paren*{\sE_{\lsc}(h) - \sE_{\lsc}^*( \sH)}^{\frac12}.
\end{equation*}

\paragraph{Example: $\ell=\Phi_{\mathrm{exp}}^{\mathrm{cstnd}}$.} Plug in $\ell=\Phi_{\mathrm{exp}}^{\mathrm{cstnd}}=\sum_{y'\neq y}\Phi_{\mathrm{exp}}\paren*{-h(x, y')}$ in \eqref{eq:sur-score}, we obtain
\begin{align*}
\sfL
 =\sum_{y'\neq y}\Phi_{\mathrm{exp}}\paren*{-h(x, y')} +\sum_{j=1}^{\num}(1-c_j(x,y))\sum_{y'\neq n+j}\Phi_{\mathrm{exp}}\paren*{-h(x, y')},
\end{align*}
where $\Phi_{\mathrm{exp}}(t)=e^{-t}$ with the constraint that $\sum_{y\in \sY}h(x,y)=0$.
By \citet[Table~3]{AwasthiMaoMohriZhong2022multi}, $\Phi_{\mathrm{exp}}^{\mathrm{cstnd}}$ admits an $\sH$-consistency bound with respect to $\ell_{0-1}$ with $\Gamma(t)=\sqrt{2t}$, using Corollary~\ref{cor:bound-score}, we obtain
\begin{equation*}
\sE_{\ldefsc}(h) - \sE_{\ldefsc}^*( \sH)
\leq \sqrt{2}\paren[\bigg]{\num + 1 - \sum_{j = 1}^{\num}\uv c_j} \paren*{\frac{\sE_{\lsc}(h) - \sE_{\lsc}^*( \sH)}{\num + 1-\sum_{j = 1}^{\num}\ov c_j}}^{\frac12}.
\end{equation*}
Since $1 \leq \num + 1 - \sum_{j = 1}^{\num}\ov
c_j\leq \num + 1 - \sum_{j = 1}^{\num}\uv c_j\leq \num + 1$, the bound can be simplified as
\begin{equation*}
\sE_{\ldefsc}(h) - \sE_{\ldefsc}^*( \sH)
\leq \sqrt{2}(\num+1)\paren*{\sE_{\lsc}(h) - \sE_{\lsc}^*( \sH)}^{\frac12}.
\end{equation*}

\paragraph{Example: $\ell=\Phi_{\rho}^{\mathrm{cstnd}}$.} Plug in $\ell=\Phi_{\rho}^{\mathrm{cstnd}}=\sum_{y'\neq y}\Phi_{\rho}\paren*{-h(x, y')}$ in \eqref{eq:sur-score}, we obtain
\begin{align*}
\sfL =\sum_{y'\neq y}\Phi_{\rho}\paren*{-h(x, y')} +\sum_{j=1}^{\num}(1-c_j(x,y))\sum_{y'\neq n+j}\Phi_{\rho}\paren*{-h(x, y')},
\end{align*}
where $\Phi_{\rho}(t)=\min\curl*{\max\curl*{0,1 - t/\rho},1}$ with the constraint that $\sum_{y\in \sY}h(x,y)=0$.
By \citet[Table~3]{AwasthiMaoMohriZhong2022multi}, $\Phi_{\rho}^{\mathrm{cstnd}}$ admits an $\sH$-consistency bound with respect to $\ell_{0-1}$ with $\Gamma(t)=t$, using Corollary~\ref{cor:bound-score}, we obtain
\begin{equation*}
\sE_{\ldefsc}(h) - \sE_{\ldefsc}^*( \sH)
\leq \sE_{\lsc}(h) - \sE_{\lsc}^*( \sH).
\end{equation*}

\section{Proof of learning bounds
  for deferral surrogate losses (Theorem~\ref{Thm:Gbound-score})}
\label{app:Gbound-score}
\GBoundScore*
\begin{proof}
  By using the standard Rademacher complexity bounds \citep{MohriRostamizadehTalwalkar2018}, for any $\delta>0$,
  with probability at least $1 - \delta$, the following holds for all $h \in \sH$:
\[
\abs*{\sE_{\sfL}(h) - \h\sE_{\sfL,S}(h)}
\leq 2 \Rad_m^{\sfL}(\sH) +
B_{\sfL} \sqrt{\tfrac{\log (2/\delta)}{2m}}.
\]
Fix $\e > 0$. By the definition of the infimum, there exists $h^* \in
\sH$ such that $\sE_{\sfL}(h^*) \leq
\sE_{\sfL}^*(\sH) + \e$. By definition of
$\h h_S$, we have
\begin{align*}
  \sE_{\sfL}(\h h_S) - \sE_{\sfL}^*(\sH)
  & = \sE_{\sfL}(\h h_S) - \h\sE_{\sfL,S}(\h h_S) + \h\sE_{\sfL,S}(\h h_S) - \sE_{\sfL}^*(\sH)\\
  & \leq \sE_{\sfL}(\h h_S) - \h\sE_{\sfL,S}(\h h_S) + \h\sE_{\sfL,S}(h^*) - \sE_{\sfL}^*(\sH)\\
  & \leq \sE_{\sfL}(\h h_S) - \h\sE_{\sfL,S}(\h h_S) + \h\sE_{\sfL,S}(h^*) - \sE_{\sfL}^*(h^*) + \e\\
  & \leq
  2 \bracket*{2 \Rad_m^{\sfL}(\sH) +
B_{\sfL} \sqrt{\tfrac{\log (2/\delta)}{2m}}} + \e.
\end{align*}
Since the inequality holds for all $\e > 0$, it implies:
\[
\sE_{\sfL}(\h h_S) - \sE_{\sfL}^*(\sH)
\leq 
4 \Rad_m^{\sfL}(\sH) +
2 B_{\sfL} \sqrt{\tfrac{\log (2/\delta)}{2m}}.
\]
Plugging in this inequality in the bound
\eqref{eq:H-consistency-bounds} completes the proof.
\end{proof}

\end{document}